\documentclass[10pt,journal]{IEEEtran}

\usepackage{amssymb}
\usepackage{amsmath}
\usepackage{multirow}
\usepackage[usenames, dvipsnames]{color}
\usepackage{epstopdf}
\usepackage{caption}
\usepackage{bm}
\usepackage{footnote}
\usepackage{stfloats}
\usepackage{placeins}
\usepackage{color,soul}
\usepackage{dsfont}
\usepackage{pbox}
\usepackage{enumerate}
\usepackage{etextools}
\usepackage{tabulary}
\usepackage{dsfont}
\usepackage{mathtools}

\usepackage{setspace}
\usepackage{algorithmic}
\usepackage[ruled,vlined]{algorithm2e}

\usepackage{subcaption}
\usepackage{caption}

\def\x{{\mathbf x}}

\def\v{{\mathbf v}}
\def\d{{\mathbf d}}

\def\y{{\mathbf y}}

\def\D{{\mathbf D}}

\def\W{{\mathbf W}}
\def\M{{\mathcal{M}}}
\def\P{{\mathcal{P}}}

\def\b{{\mathbf b}}
\def\S{{\mathcal{S}}}

\def\C{{\mathbf C}}

\def\PP{{\mathbf P}}

\def\gama{{\boldsymbol \gamma}}

\def\Delt{{\boldsymbol \Delta}}

\def\O{{\boldsymbol \Omega}}
\def\lamda{{\boldsymbol \lambda}}
\def\vps{{\bm{\mathbf{\mathcal{E}}}}}

\def\Loi{{\ell_{0,\infty}}}

\def\DCPE{{\text{DCP}_\lamda^{\hspace{0.04cm} \vps}}}

\def\PM{{\P_{\M_\lamda}}}

\def\pp{{\scriptscriptstyle{\PP}}}

\usepackage{amsthm}
\theoremstyle{plain}
\newtheorem{thm}{Theorem} 
\newenvironment{customthm}[1]
{\renewcommand\thethm{#1}\thm}
{\endthm}

\theoremstyle{definition}
\newtheorem{defn}[thm]{Definition} 

\newtheorem{lemma}{Lemma}

\usepackage{xr}
\begin{document}

\title{Multi-Layer Convolutional Sparse Modeling:\\ Pursuit and Dictionary Learning}
%
\date{}
\author{Jeremias~Sulam,~\IEEEmembership{Member, IEEE,}
        Vardan~Papyan,~
        Yaniv~Romano,~
        and~Michael~Elad~\IEEEmembership{Fellow, IEEE} %
        \thanks{\noindent J. Sulam, and M. Elad are with the Computer Science Department, Technion-Israel Institute of Technology. Y. Romano and V. Papyan are with the  Statistics Department of Stanford University. 
        } }

\graphicspath{{./Figures/}}

\IEEEtitleabstractindextext{
\begin{abstract} 
The recently proposed Multi-Layer Convolutional Sparse Coding (ML-CSC) model, consisting of a cascade of convolutional sparse layers, provides a new interpretation of Convolutional Neural Networks (CNNs). Under this framework, the forward pass in a CNN is equivalent to a pursuit algorithm aiming to estimate the nested sparse representation vectors from a given input signal. Despite having served as a pivotal connection between CNNs and sparse modeling, a deeper understanding of the ML-CSC is still lacking. In this work, we propose a sound pursuit algorithm for the ML-CSC model by adopting a projection approach. We provide new and improved bounds on the stability of the solution of such pursuit and we analyze different practical alternatives to implement this in practice. We show that the training of the filters is essential to allow for non-trivial signals in the model, and we derive an online algorithm to learn the dictionaries from real data, effectively resulting in cascaded sparse convolutional layers. Last, but not least, we demonstrate the applicability of the ML-CSC model for several applications in an unsupervised setting, providing competitive results. Our work represents a bridge between matrix factorization, sparse dictionary learning and sparse auto-encoders, and we analyze these connections in detail.
\end{abstract}

}

\maketitle

\IEEEdisplaynontitleabstractindextext

\IEEEpeerreviewmaketitle


\section{Introduction}

New ways of understanding real world signals, and proposing ways to model their intrinsic properties, have led to improvements in signal and image restoration, detection and classification, among other problems. Little over a decade ago, sparse representation modeling brought about the idea that natural signals can be (well) described as a linear combination of only a few building blocks or components, commonly known as atoms \cite{Bruckstein2009}. Backed by elegant theoretical results, this model led to a series of works dealing either with the problem of the pursuit of such decompositions, or with the design and learning of better atoms from real data \cite{Rubinstein2010_dict}. The latter problem, termed dictionary learning, empowered sparse enforcing methods to achieve remarkable results in many different fields from signal and image processing \cite{Sulam2014,Romano2014,Mairal2009} to machine learning \cite{Jiang2013,patel2014dictionaries,shrivastava2014multiple}.

Neural networks, on the other hand, were introduced around forty years ago and were shown to provide powerful classification algorithms through a series of function compositions \cite{lecun1990handwritten,rumelhart1988learning}.  It was not until the last half-decade, however, that through a series of incremental modifications these methods were boosted to become the state-of-the-art machine learning tools for a wide range of problems, and across many different fields \cite{lecun2015deep}. For the most part, the development of new variants of deep convolutional neural networks (CNNs) has been driven by trial-and-error strategies and a considerable amount of intuition. 

Withal, a few research groups have begun providing theoretical justifications and analysis strategies for CNNs from very different perspectives. For instance, by employing wavelet filters instead of adaptive ones, the work by Bruna and Mallat \cite{bruna2013invariant} demonstrated how \emph{scattering networks} represent shift invariant analysis operators that are robust to deformations (in a Lipschitz-continuous sense). The inspiring work of \cite{patel2015probabilistic} proposed a generative Bayesian model, under which typical deep learning architectures perform an inference process. In \cite{cohen16Shashua}, the authors proposed a hierarchical tensor factorization analysis model to analyze deep CNNs. Fascinating connections between sparse modeling and CNN have also been proposed. In \cite{gregor2010learning}, a neural network architecture was shown to be able to learn iterative shrinkage operators, essentially \emph{unrolling} the iterations of a sparse pursuit. Building on this interpretation, the work in \cite{xin2016maximal} further showed that CNNs can in fact improve the performance of sparse recovery algorithms. 

A precise connection between sparse modeling and CNNs was recently presented in \cite{Papyan2016convolutional}, and its contribution is centered in defining the Multi-Layer Convolutional Sparse Coding (ML-CSC) model. When deploying this model to real signals, compromises were made in way that each layer is only \emph{approximately} explained by the following one. With this relaxation in the pursuit of the convolutional representations, the main observation of that work is that the inference stage of CNNs -- nothing but the forward-pass -- can be interpreted as a very crude pursuit algorithm seeking for unique sparse representations. This is a useful perspective as it provides a precise optimization objective which, it turns out, CNNs attempt to minimize.  

The work in \cite{Papyan2016convolutional} further proposed improved pursuits for approximating the sparse representations of the network, or feature maps, such as the Layered Basis Pursuit algorithm. Nonetheless, as we will show later, neither this nor the forward pass serve the ML-CSC model exactly, as they do not provide signals that comply with the model assumptions. In addition, the theoretical guarantees accompanying these layered approaches suffer from bounds that become looser with the network's depth. The lack of a suitable pursuit, in turn, obscures how to properly sample from the ML-CSC model, and how to train the dictionaries from real data.

In this work we undertake a fresh study of the ML-CSC and of pursuit algorithms for signals in this model. Our contributions will be guided by addressing the following questions:

\begin{enumerate}

	\item Given proper convolutional dictionaries, how can one project\footnote{By projection, we refer to the task of getting the closest signal to the one given that obeys the model assumptions.} signals onto the ML-CSC model?

	\item When will the model allow for \emph{any} signal to be expressed in terms of nested sparse representations? In other words, is the model empty?

	\item What conditions should the convolutional dictionaries satisfy? and how can we adapt or learn them to represent real-world signals?

	\item How is the learning of the ML-CSC model related to traditional CNN and dictionary learning algorithms? 

	\item What kind of performance can be expected from this model?

\end{enumerate}

The model we analyze in this work is related to several recent contributions, both in the realm of sparse representations and deep-learning. On the one hand, the ML-CSC model is tightly connected to dictionary constrained learning techniques, such as Chasing Butterflies approach \cite{Lemagoarou15}, fast transform learning \cite{Chabiron2013}, Trainlets \cite{Sulam2016}, among several others. On the other hand, and because of the unsupervised flavor of the learning algorithm, our work shares connections to sparse auto-encoders \cite{ng2011sparse}, and in particular to the k-sparse \cite{makhzani2013k} and winner-take-all versions \cite{makhzani2015winner}.

In order to progressively answer the questions posed above, we will first review the ML-CSC model in detail in Section \ref{sec:Background}. We will then study how signals can be projected onto the model in Section \ref{sec:Projection}, where we will analyze the stability of the projection problem and provide theoretical guarantees for practical algorithms. We will then propose a learning formulation in Section \ref{sec:Learning}, which will allow, for the first time, to obtain a trained ML-CSC model from real data while being perfectly faithful to the model assumptions. In this work we restrict our study to the learning of the model in an unsupervised setting. This approach will be further demonstrated on signal approximation and unsupervised learning applications in Section \ref{sec:experiments}, before concluding in Section \ref{sec:conclusions}.


\section{Background}
\label{sec:Background}

\subsection{Convolutional Sparse Coding}

\begin{figure}
\begin{center}
	\includegraphics[trim = 0 0 0 0, width = .48\textwidth]{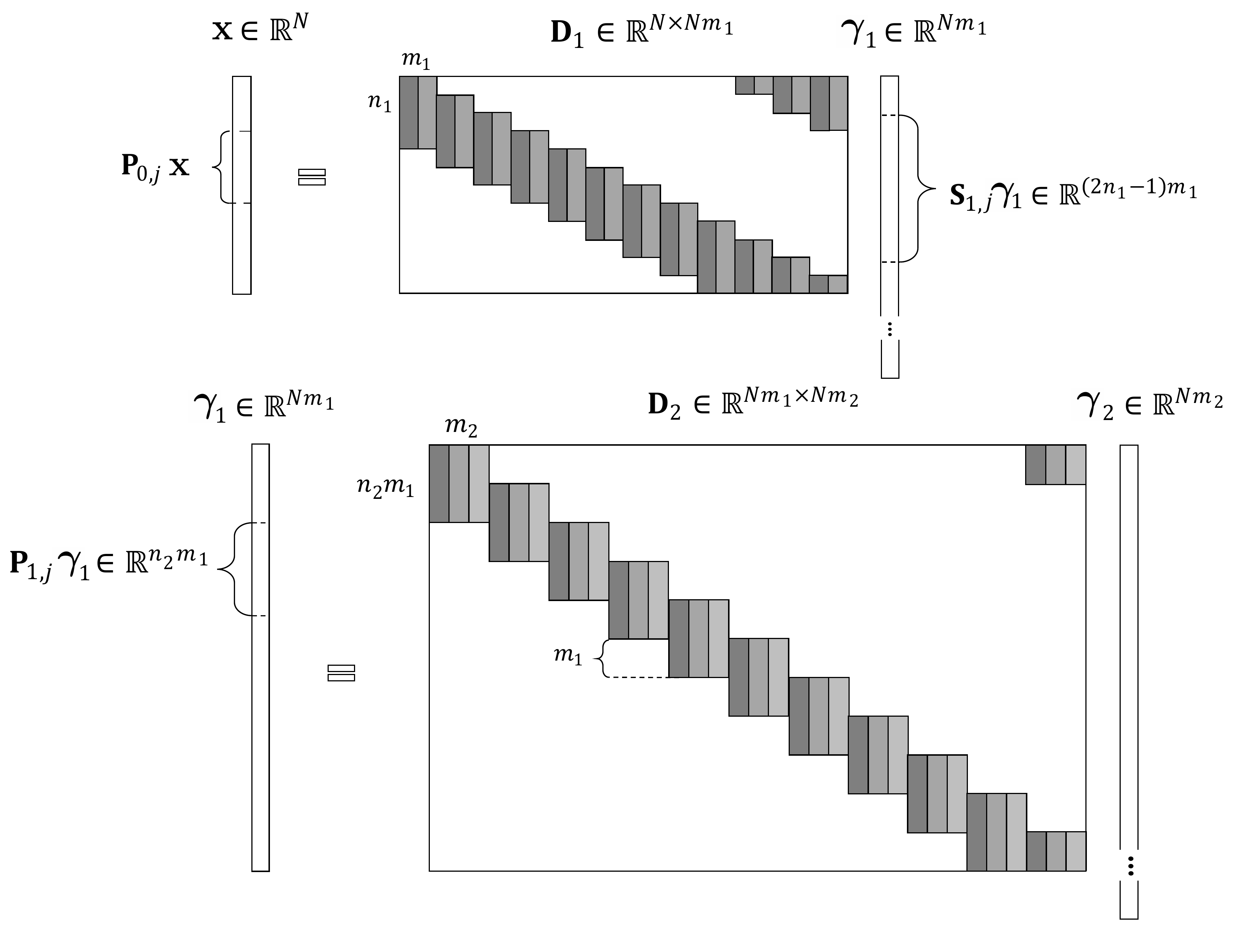}
	\caption{The CSC model (top), and its ML-CSC extension by imposing a similar model on $\gama_1$ (bottom). 
	}
	\label{Fig:ML-CSCmodel}
\end{center}
\end{figure}

The Convolutional Sparse Coding (CSC) model assumes a signal $\x \in \mathbb{R}^N$ admits a decomposition as $\D_1\gama_1$, where $\gama_1 \in \mathbb{R}^{Nm_1}$ is sparse and $\D_1 \in \mathbb{R}^{N\times Nm_1}$ has a convolutional structure. More precisely, this dictionary consists of $m_1$ local $n_1$-dimensional filters at every possible location (Figure \ref{Fig:ML-CSCmodel} top).
An immediate consequence of this model assumption is the fact that each $j^{th}$ \emph{patch} $\mathbf{P}_{0,j}\x \in \mathbb{R}^{n_1}$ from the signal $\x$ can be expressed in terms of a shift-invariant local model corresponding to a \emph{stripe} from the global sparse vector, $\mathbf{S}_{1,j}\gama_1 \in \mathbb{R}^{(2n_1-1)m_1}$. From now on, and for the sake of simplicity, we will drop the first index on the stripe and patch extraction operators, simply denoting the $j^{th}$ stripe from $\gama_1$ as $\mathbf{S}_j\gama_1$. 

In the context of CSC, the sparsity of the representation is better captured through the $\Loi$ pseudo-norm \cite{WorkingLocallyThinkingGlobally}. This measure, as opposed to the traditional $\ell_0$, provides a notion of local sparsity and it is defined by the maximal number of non-zeros in a stripe from $\gama$. Formally,
\begin{equation}
	\|\gama\|^s_{0,\infty} = \underset{i}{\max}\ \| \mathbf{S}_i \gama \|_0.
\end{equation}
We kindly refer the reader to \cite{WorkingLocallyThinkingGlobally} for a more detailed description of this model, as well as extensive theoretical guarantees associated with the model stability and the success of pursuit algorithms serving it. 

This model presents several characteristics that make it relevant and interesting. On the one hand, CSC provides a systematic and formal way to develop and analyze very popular and successful patch-based algorithms in signal and image processing \cite{WorkingLocallyThinkingGlobally}. From a more practical perspective, on the other hand, the convolutional sparse model has recently received considerable attention in the computer vision and machine learning communities. Solutions based on the CSC have been proposed for detection \cite{sermanet2013pedestrian}, compressed sensing \cite{li2013convolutional} texture-cartoon separation \cite{zhang2016convolutional,zhang2018convolutional}, inverse problems \cite{Papyan_2017_ICCV,Heide2015,choudhury2017consensus} and feature learning \cite{henaff2011unsupervised,szlam2011structured}, and different convolutional dictionary learning algorithms have been proposed and analyzed \cite{Papyan_2017_ICCV,Wohlberg2016,liu2017online}. Interestingly, this model has also been employed in a hierarchical way \cite{zeiler2010deconvolutional,szlam2010convolutional,kavukcuoglu2010learning,he2014unsupervised} mostly following intuition and imitating successful CNNs' architectures. This connection between convolutional features and multi-layer constructions was recently made precise in the form of the Multi-Layer CSC model, which we review next.

\subsection{Multi Layer CSC}
\label{sec:MultiLayerCSC}
The Multi-Layer Convolutional Sparse Coding (ML-CSC) model is a natural extension of the CSC described above, as it assumes that a signal can be expressed by sparse representations at different layers in terms of nested convolutional filters. Suppose $\x = \D_1\gama_1$, for a convolutional dictionary $\D_1 \in \mathbb{R}^{N\times Nm_1}$ and an $\Loi$-sparse representation $\gama_1 \in \mathbb{R}^{Nm_1}$.  One can cascade this model by imposing a similar assumption on the representation $\gama_1$, i.e., $\gama_1 = \D_2\gama_2$, for a corresponding convolutional dictionary $\D_2\in\mathbb{R}^{Nm_1\times Nm_2}$ with $m_2$ local filters and a $\Loi$-sparse $\gama_2$, as depicted in Figure \ref{Fig:ML-CSCmodel}. In this case, $\D_2$ is a also a convolutional dictionary with local filters skipping $m_1$ entries at a time\footnote{This construction provides operators that are convolutional in the space domain, but not in the channel domain -- just as for CNNs.} -- as there are $m_1$ \emph{channels} in the representation $\gama_1$. 

Because of this multi-layer structure, vector $\gama_1$ can be viewed both as a sparse representation (in the context of $\x=\D_1\gama_1$) or as a signal (in the context of $\gama_1 = \D_2\gama_2$). Thus, one one can refer to both its stripes (looking backwards to patches from $\x$) or its patches (looking forward, corresponding to stripes of $\gama_2$). In this way, when analyzing the ML-CSC model we will not only employ the $\ell_{0,\infty}$ norm as defined above, but we will also leverage its \emph{patch} counterpart, where the maximum is taken over all patches from the sparse vector by means of a patch extractor operator $\mathbf{P}_i$. In order to make their difference explicit, we will denote them as $\|\gama\|^s_{0,\infty}$ and $\|\gama\|^p_{0,\infty}$ for stripes and patches, respectively. In addition, we will employ the $\ell_{2,\infty}$ norm version, naturally defined as $\|\gama\|^s_{2,\infty} = \underset{i}{\max}\ \|\mathbf{S}_i\gama \|_2$, and analogously for patches.	

We now formalize the model definition:

\begin{defn}{ML-CSC model:}\\
Given a set of convolutional dictionaries $\{\D_i\}_{i=1}^L$ of appropriate dimensions, a signal $\x(\gama_i) \in \mathbb{R}^{N}$ admits a representation in terms of the ML-CSC model, i.e. $\x(\gama_i) \in \M_{\lamda}$, if
	\begin{align*}
	 \x = \D_1 \gama_1, &  \quad \|\gama_1\|^s_{0,\infty} \leq \lambda_1, \\
	\gama_1 = \D_2 \gama_2, & \quad \|\gama_2\|^s_{0,\infty} \leq \lambda_2 , \\
	\phantom{..} \vdots  \\
	\gama_{L-1} = \D_L \gama_L, & \quad \|\gama_L\|^s_{0,\infty} \leq \lambda_L.
	\end{align*}
\end{defn}

Note that $\x(\gama_i) \in \M_\lamda$ can also be expressed as $\x = \D_1\D_2\dots\D_L \gama_L$. We refer to $\D^{(i)}$ as the \emph{effective} dictionary at the $i^{th}$ level, i.e., $\D^{(i)} = \D_1\D_2\dots\D_i$. This way, one can concisely write
\begin{equation}
	\x = \D^{(i)}\gama_i, \ 1\leq i\leq L.
\end{equation}

Interestingly, the ML-CSC can be interpreted as a special case of a CSC model: one that enforces a very specific structure on the intermediate representations. We make this statement precise in the following Lemma:

\begin{lemma}{} \label{lemma:MLCSCisCSC}
	Given the ML-CSC model described by the set of convolutional dictionaries $\{\D_i\}_{i=1}^L$, with filters of spatial dimensions $n_i$ and channels $m_i$, any dictionary $\D^{(i)} = \D_1 \D_2 \dots \D_i$ is a convolutional dictionary with $m_i$ local atoms of dimension $n_i^{\text{eff}} = \sum_{j=1}^{i} n_j - (i-1)$. In other words, the ML-CSC model is a structured global convolutional model.
\end{lemma}
\noindent
The proof of this lemma is rather straight forward, and we include it in the Supplementary Material \ref{app:MLCSCisCSC}.
Note that what was denoted as the effective dimension at the $i^{th}$ layer is nothing else than what is known in the deep learning community as the \emph{receptive field} of a filter at layer $i$. Here, we have made this concept precise in the context of the ML-CSC model.

As it was presented, the convolutional model assumes that every $n$-dimensional atom is located at every possible location, which implies that the filter is shifted with strides of $s=1$. An alternative, which effectively reduces the redundancy of the resulting dictionary, is to consider a stride greater than one. In such case, the resulting dictionary is of size $N\times Nm_1/s$ for one dimensional signals, and $N\times N m_1 / s^2$ for images. This construction, popular in the CNN community, does not alter the effective size of the filters but rather decreases the length of each stripe by a factor of $s$ in each dimension. In the limit, when $s = n_1$, one effectively considers non-overlapping blocks and the stripe will be of length\footnote{When $s=n_1$, the system is no longer shift-invariant, but rather invariant with a shift of $n$ samples.} $m_1$ -- the number of local filters. Naturally, one can also employ $s>1$ for any of the multiple layers of the ML-CSC model. We will consider $s=1$ for all layers in our derivations for simplicity.

\begin{figure}
\begin{center}
	\includegraphics[trim = 0 10 0 10, width = .46\textwidth]{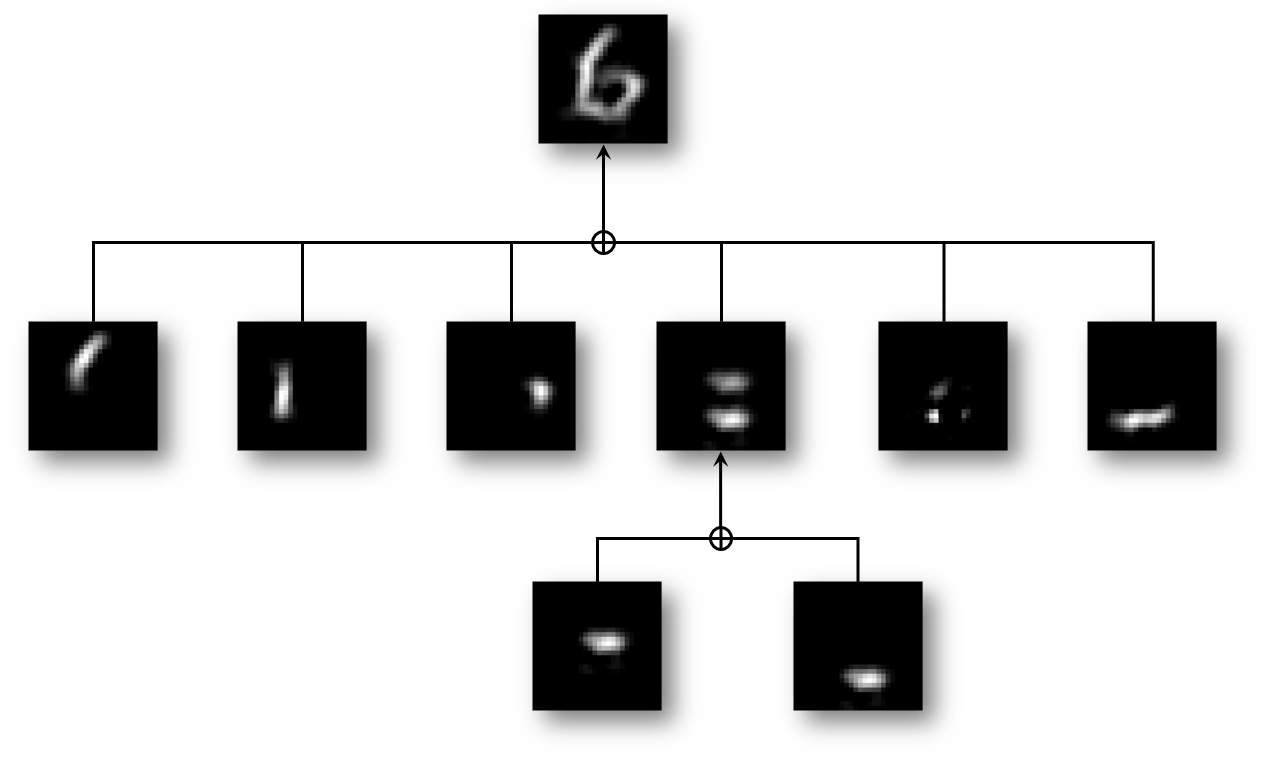}
	\caption{From atoms to molecules: Illustration of the ML-CSC model for a number 6. Two local convolutional atoms (bottom row) are combined to create slightly more complex structures -- molecules -- at the second level, which are then combined to create the global atom representing, in this case, a digit.
	}
	\label{Fig:atomDecomposition}
\end{center}
\end{figure}

The ML-CSC imposes a unique structure on the global dictionary $\D^{(L)}$, as it provides a multi-layer linear composition of simpler structures. In other words, $\D_1$ contains (small) local $n_1$-dimensional atoms. The product $\D_1\D_2$ contains in each of its columns a linear combination of atoms from $\D_1$, merging them to create molecules. Further layers continue to create more complex constructions out of the simpler convolutional building blocks. We depict an example of such decomposition in Figure \ref{Fig:atomDecomposition} for a $3^{rd}$-layer convolutional atom of the digit ``6''. While the question of how to obtain such dictionaries will be addressed later on, let us make this illustration concrete: consider this atom to be given by $\x_0 = \D_1\D_2\d_3$, where $\d_3$ is sparse, producing the upper-most image $\x_0$. Denoting by $\mathcal{T}(\d_3) = Supp(\d_3)$, this atom can be equally expressed as 
\begin{equation}
	\x_0 = \D^{(2)}\d_3 = \sum_{j\in \mathcal{T}(\d_3)} \d^{(2)}_j d^j_{3}.
\end{equation}
In words, the effective atom is composed of \emph{a few} elements from the effective dictionary $\D^{(2)}$. These are the building blocks depicted in the middle of Figure \ref{Fig:atomDecomposition}. Likewise, focusing on the fourth of such atoms, $\d^{(2)}_{j_4} = \D_1 \d_{2,j_4}$. In this particular case, $\|\d_{2,j_4}\|_0 = 2$, so we can express 
$\d^{(2)}_{j_4} =  \d^{(1)}_{i_1} d^{i_1}_{2,j_1} + \d^{(1)}_{i_2} d^{i_2}_{2,j_1}$.
	%
These two atoms from $\D_1$ are precisely those appearing in the bottom of the decomposition.



\subsection{Pursuit in the noisy setting}

Real signals might contain noise or deviations from the above idealistic model assumption, preventing us from enforcing the above model exactly. Consider the scenario of acquiring a signal $\y = \x + \v$, where $\x \in \mathcal{M}_\lamda$ and $\v$ is a nuisance vector of bounded energy, $\|\v\|_2 \leq \mathcal{E}_0$. In this setting, the objective is to estimate all the representations $\gama_i$ which explain the measurements $\y$ up to an error of $\mathcal{E}_0$. 
In its most general form, this pursuit is represented by the Deep Coding Problem ($\DCPE$), as introduced in \cite{Papyan2016convolutional}:

\begin{defn}{$\DCPE$ Problem:}\\
	For a global signal $\y$, a set of convolutional dictionaries $\{ \D_i \}_{i=1}^L$, and vectors $\lamda$ and $\vps$:
	\begin{align*}
	 (\DCPE): \quad \quad \text{find} \quad \{\gama_i\}_{i=1}^{L} & \quad \text{ s.t. } & \\[.1cm]
	\| \y - \D_1 \gama_1\|_2 & \leq \mathcal{E}_0, & \qquad \| \gama_1 \|^s_{0,\infty}  & \leq \lambda_1 &\\
	\| \gama_1 - \D_2 \gama_2 \|_2 & \leq \mathcal{E}_1,& \qquad  \| \gama_2 \|^s_{0,\infty}  & \leq \lambda_2 &\\
	 &\phantom{..} \vdots & \phantom{..} &\vdots  \\
	\| \gama_{L-1} - \D_L \gama_L \|_2  &\leq \mathcal{E}_{L-1},& \qquad \| \gama_{L} \|^s_{0,\infty} &\leq \lambda_L &
	\end{align*}
	where $\lambda_i$ and $\mathcal{E}_i$ are the $i^{th}$ entries of $\lamda$ and $\vps$, respectively.
\end{defn}

{
The solution to this problem was shown to be stable in terms of a bound on the $\ell_2$-distance between the estimated representations $\hat{\gama}_i$ and the true ones, $\gama_i$. These results depend on the characterization of the dictionaries through their mutual coherence, $\mu(\D)$, which measures the maximal normalized correlation between atoms in the dictionary. Formally, assuming the atoms are normalized as $\|\d_i\|_2 = 1\ \forall i$, this measure is defined as 
\begin{equation}
	\mu(\D) = \underset{i\neq j}{\max}\ | \d_i^T\d_j|.
\end{equation}
Relying on this measure, Theorem 5 in \cite{Papyan2016convolutional} shows that given a signal $\x(\gama_i)\in\PM$ contaminated with noise of known energy $\mathcal{E}^2_0$, if the representations satisfy the sparsity constraint
\begin{equation}
	\| \gama_i \|^s_{0,\infty} < \frac{1}{2} \left( 1 + \frac{1}{\mu(\D_i)} \right),
\end{equation}
then the solution to the $\DCPE$ given by $\{\hat{\gama}_i\}_{i=1}^L$ satisfies
\begin{equation}
	\| \gama_i- \hat{\gama}_i \|_2^2 \leq 4 {\mathcal{E}_0}^2 \prod_{j=1}^{i} \frac{4^{i-1}}{1-(2\|\gama_j\|^s_{0,\infty}-1)\mu(\D_j)}.
\end{equation}
\noindent
In the particular instance of the $\DCPE$ where $\mathcal{E}_i = 0$ for $1\leq i \leq L-1$, the above bound can be made tighter by a factor of $4^{i-1}$ while preserving the same form.

These results are encouraging, as they show for the first time stability guarantees for a problem for which the forward pass provides an approximate solution. More precisely, if the above model deviations are considered to be greater than zero ($\mathcal{E}_i > 0$) several layer-wise algorithms, including the forward pass of CNNs, provide approximations to the solution of this problem \cite{Papyan2016convolutional}. }
We note two remarks about these stability results:
\begin{enumerate}
	\item The bound increases with the number of layers or the depth of the network. This is a direct consequence of the layer-wise relaxation in the above pursuit, which causes these discrepancies to accumulate over the layers.

	\item Given the underlying signal $\x(\gama_i) \in \M_{\lamda}$, with representations $\{\gama_i\}_{i=1}^L$, this problem searches for their corresponding estimates $\{\hat{\gama}_i\}_{i=1}^L$. However, because at each layer $\| \hat{\gama}_{i-1} - \D_i \hat{\gama}_i \|_2 > 0$, this problem \emph{does not} provide representations for a signal in the model. In other words, $\hat{\x} \neq \D_1 \hat{\gama}_1$, $\hat{\gama}_1 \neq \D_2 \hat{\gama}_2$, and generally $\hat{\x} \notin \M_{\lamda}$.
\end{enumerate}


\section{A Projection Alternative}
\label{sec:Projection}

In this section we provide an alternative approach to the problem of estimating the underlying representations $\gama_i$ under the same noisy scenario of $\y = \x(\gama_i) + \v$. In particular, we are interested in projecting the measurements $\y$ onto the set $\M_\lamda$. 
Consider the following projection problem:

\begin{defn}{ML-CSC Projection $\P_{\M_\lamda}$:}\\
	For a signal $\y$ and a set of convolutional dictionaries $\{ \D_i \}_{i=1}^L$, define the Multi-Layer Convolutional Sparse Coding projection as:
	\begin{equation}
	 (\PM): \quad \min_{\{\gama_i\}_{i=1}^L} \quad \| \y - \x(\gama_i) \|_2 \quad \text{ s.t. } \quad \x(\gama_i) \in \M_\lamda.
	\label{Eq:PMproblem}
	\end{equation}
\end{defn}
\noindent
Note that this problem differs from the $\DCPE$ counterpart in that we seek for a signal close to $\y$, whose representations $\gama_i$ give rise to $\x(\gama_i) \in \M_{\lamda}$. This is more demanding (less general) than the formulation in the $\DCPE$. 
Put differently, the $\PM$ problem can be considered 
as a special case of the $\DCPE$ where model deviations are allowed only at the outer-most level. 
Recall that the theoretical analysis of the $\DCPE$ problem indicated that the error thresholds should increase with the layers. Here, the $\PM$ problem suggests a completely different approach.

\subsection{Stability of the projection $\PM$}

Given $\y = \x(\gama_i) + \v$, one can seek for the underlying representations $\gama_i$ through either the $\DCPE$ or $\PM$ problem. In light of the above discussion and the known stability result for the $\DCPE$ problem, how close will the solution of the $\PM$ problem be from the true set of representations? The answer is provided through the following result.



\begin{thm}{ Stability of the solution to the $\PM$ problem:} \label{Thm:GlobalStability} \\
	Suppose $\x(\gama_i) \in \M_\lamda$ is observed through $\y = \x+ \v$, where $\v$ is a bounded noise vector, $\|\v\|_2 \leq \mathcal{E}_0$, and 
	$\|\gama_i\|^s_{0,\infty} = \lambda_i < \frac{1}{2}\left(1+\frac{1}{\mu(\D^{(i)})}\right)$, for $1\leq i \leq L$. Consider the set $\{\hat{\gama}_i\}_{i=1}^{L}$ to be the solution of the $\PM$ problem. Then,
	\begin{equation} \label{Eq:DCPEStability}
	\| \gama_i- \hat{\gama}_i \|_2^2 \leq \frac{4\mathcal{E}_{0}^2}{1-(2\|\gama_{i}\|^s_{0,\infty}-1)\mu(\D^{(i)})}.
	\end{equation}
\end{thm}
For the sake of brevity, we include the proof of this claim in the Supplementary Material \ref{app:StabilityforPM}. However, we note a few remarks:
\begin{enumerate}

	\item The obtained bounds are not cumulative across the layers. In other words, they do not grow with the depth of the network.
 
	\item Unlike the stability result for the $\DCPE$ problem, the assumptions on the sparse vectors $\gama_i$ are given in terms of the mutual coherence of the effective dictionaries $\D^{(i)}$. Interestingly enough, we will see in the experimental section that one can in fact have that $\mu(\D^{(i-1)}) > \mu(\D^{(i)})$ in practice; i.e., the effective dictionary becomes incoherent as it becomes deeper. Indeed, the deeper layers provide larger atoms with correlations that are expected to be lower than the inner products between two small local (and overlapping) filters.

	\item While the conditions imposed on the sparse vectors $\gama_i$ might seem prohibitive, one should remember that this follows from a worst case analysis. Moreover, one can effectively construct analytic nested convolutional dictionaries with small coherence measures, as shown in \cite{Papyan2016convolutional}.

\end{enumerate}

Interestingly, one can also formulate bounds for the stability of the solution, i.e. \mbox{$\|\gama_i - \hat{\gama}_i\|_2^2$}, which are the tightest for the inner-most layer, and then increase as one moves to shallower layers -- precisely the opposite behavior of the solution to the $\DCPE$ problem. This result, however, provides bounds that are generally looser than the one presented in the above theorem, and so we defer this to the Supplementary Material.

\subsection{Pursuit Algorithms}
\label{sec:Pursuit_Algorithms}

We now focus on the question of how one can solve the above problems in pracice.
As shown in \cite{Papyan2016convolutional}, one can approximate the solution to the $\DCPE$ in a layer-wise manner, solving for the sparse representations $\hat{\gama}_i$ progressively from $i=1,\dots,L$. Surprisingly, the Forward Pass of a CNN \emph{is} one such algorithm, yielding stable estimates. 
The Layered BP algorithm was also proposed, where each representation $\hat{\gama}_i$ is sparse coded (in a Basis Pursuit formulation) given the previous representation $\hat{\gama}_{i-1}$ and dictionary $\D_i$. As solutions to the $\DCPE$ problem, these algorithms inherit the layer-wise relaxation referred above, which causes the theoretical bounds to increase as a function of the layers or network depth.


Moving to the variation proposed in this work, how can one solve the $\PM$ problem in practice? Applying the above layer-wise pursuit is clearly not an option, since after obtaining a necessarily distorted estimate $\hat{\gama}_1$ we cannot proceed with equalities for the next layers, as $\gama_1$ does not necessarily have a perfectly sparse representation with respect to $\D_2$. Herein we present a simple approach based on a global sparse coding solver which yields provable stable solutions. 
\begin{algorithm}[h]
\setstretch{1.5}
	\textbf{Input:} $\y,\{\D_i\},k$\;

	$\hat{\gama}_L \leftarrow \text{Pursuit}(\y,\D^{(L)},k) $\;
	
	\For{$j = L,\dots,1$}{
		$\hat{\gama}_{j-1} \leftarrow \D_j\hat{\gama}_j $
		}

	\Return $\{\hat{\gama}_i\}$\;
	\caption{ML-CSC Pursuit}
	\label{Alg:MulilayerPursuit}
\end{algorithm}

Consider Algorithm \ref{Alg:MulilayerPursuit}. This approach circumvents the problem of sparse coding the intermediate features while guaranteeing their exact expression in terms of the following layer. This is done by first running a Pursuit for the deepest representation through an algorithm which provides an approximate solution to the following problem:
\begin{equation}\label{eq:PursuitDeepest}
	\underset{\gama}{\min}\ \|\y - \D^{(L)} \gama \|_2^2\ \text{ s.t. } \|\gama\|^s_{0,\infty} \leq k.
\end{equation}

Once the deepest representation has been estimated, we proceed by obtaining the remaining ones by simply applying their definition, thus assuring that $\hat{\x} = \D^{(i)}\hat{\gama}_i \in \M_\lambda$. While this might seem like a dull strategy, we will see in the next section that, if the measurements $\y$ are close enough to a signal in the model, Algorithm \ref{Alg:MulilayerPursuit} indeed provides stable estimates $\hat{\gama}_i$. In fact, the resulting stability bounds will be shown to be generally tighter than those existing for the layer-wise pursuit alternative. Moreover, as we will later see in the Results section, this approach can effectively be harnessed in practice in a real-data scenario.

\subsection{Stability Guarantees for Pursuit Algorithms}
\label{sec:StabilityPursuits}
Given a signal $\y = \x(\gama_i) + \v$, and the respective solution of the ML-CSC Pursuit in Algorithm \ref{Alg:MulilayerPursuit}, how close will the estimated $\hat{\gama}_i$ be to the original representations $\gama_i$? These bounds will clearly depend on the specific Pursuit algorithm employed to obtain $\hat{\gama}_L$. In what follows, we will present two stability guarantees that arise from solving this sparse coding problem under two different strategies: a greedy and a convex relaxation approach. Before diving in, however, we present two elements that will become necessary for our derivations. 

The first one is a property that relates to the propagation of the support, or non-zeros, across the layers. Given the support of a sparse vector $\mathcal{T} = Supp(\gama)$, consider dictionary $\D_\mathcal{T}$ as the matrix containing only the columns indicated by $\mathcal{T}$. Define $\|\D_\mathcal{T}\|^0_{\infty} = \sum_{i=1}^{n} \|\mathcal{R}_i \D_\mathcal{T}\|^0_{\infty}$, where $\mathcal{R}_i$ extracts the $i^{th}$ \emph{row} of the matrix on its right-hand side. In words, $\|\D_\mathcal{T}\|^0_{\infty}$ simply counts the number of non-zero rows of $\D_\mathcal{T}$. With it, we now define the following property:

\begin{defn}{Non Vanishing Support (N.V.S.):}\\
	A sparse vector $\gama$ with support $\mathcal{T}$ satisfies the N.V.S property for a given dictionary $\D$ if 
	\begin{equation}
		\| \D \gama \|_0 = \|\D_\mathcal{T}\|^0_{\infty}.
	\end{equation}
	\label{def:N.V.S.Property}
\end{defn}
Intuitively, the above property implies that the entries in $\gama$ will not cause two or more atoms to be combined in such a way that (any entry of) their supports cancel each other. Notice that this is a very natural assumption to make. 
Alternatively, one could assume the non-zero entries from $\gama$ to be Gaussian distributed, and in this case the N.V.S. property holds \emph{a.s.}

A direct consequence of the above property is that of maximal cardinality of representations. If $\gama$ satisfies the N.V.S property for a dictionary $\D$, and $\bar{\gama}$ is another sparse vector with equal support (i.e., $Supp(\gama) = Supp(\bar{\gama})$), then necessarily $Supp(\D\bar{\gama}) \subseteq Supp(\D\gama)$, and thus $\|\D\gama\|_0 \geq \|\D\bar{\gama}\|_0$. This follows from the fact that the number of non-zeros in $\D\bar{\gama}$ cannot be greater than the sum of non-zero rows from the set of atoms, $\D_\mathcal{T}$. 

The second element concerns the local stability of the Stripe-RIP, the convolutional version of the Restricted Isometric Property \cite{Candes2005}. As defined in \cite{WorkingLocallyThinkingGlobally}, a convolutional dictionary $\D$ satisfies the Stripe-RIP condition with constant $\delta_k$ if, for every $\gama$ such that $\|\gama\|^s_{0,\infty}=k$,
\begin{equation} \label{Eq:SRIP}
	(1-\delta_k)\|\gama\|^2_2 \leq \|\D\gama\|^2_2 \leq (1+\delta_k)\|\gama\|^2_2.
\end{equation}
The S-RIP bounds the maximal change in (global) energy of a $\Loi$-sparse vector when multiplied by a convolutional dictionary. We would like to establish an equivalent property but in a local sense. Recall the $\|\x\|^p_{2,\infty}$ norm, given by the maximal norm of a \emph{patch} from $\x$, i.e. $\|\x\|^p_{2,\infty} = \underset{i}{\max} \|\mathbf{P}_i\x\|_2$. Analogously, one can consider $\|\gama\|^s_{2,\infty} = \underset{i}{\max}\ \|\mathbf{S}_i\gama\|_2$ to be the maximal norm of a \emph{stripe} from $\gama$.

Now, is $\|\D\gama\|^{p}_{2,\infty}$ nearly isometric? The (partially affirmative) answer is given in the form of the following Lemma, which we prove in the Supplementary Material \ref{app:LocalStabilitySRIP}.

\begin{lemma}{Local one-sided near isometry property:} \label{lemma:LocalSRIP} \\
If $\D$ is a convolutional dictionary satisfying the Stripe-RIP condition in \eqref{Eq:SRIP} with constant $\delta_k$, then
\begin{equation}
	\|\D\gama\|^{2,p}_{2,\infty} \leq (1+\delta_k)\ \|\gama\|^{2,s}_{2,\infty}.
\end{equation}
\end{lemma}

This result is worthy in its own right, as it shows for the first time that not only the CSC model is globally stable for $\Loi$-sparse signals, but that one can also bound the change in energy in a local sense (in terms of the $\ell_{2,\infty}$ norm). 
While the above Lemma only refers to the upper bound of $\|\D\gama\|^{2,p}_{2,\infty}$, we conjecture that an analogous lower bound can be shown to hold as well. 

With these elements, we can now move to the stability of the solutions provided by Algorithm \ref{Alg:MulilayerPursuit}:

\begin{thm}{Stable recovery of the Multi-Layer Pursuit Algorithm in the convex relaxation case:} \label{Thm:StabilityPursuitLasso} \\
	Suppose a signal $\x(\gama_i) \in \M_\lamda$ is contaminated with locally-bounded noise $\v$, resulting in $\y = \x + \v$, $\|\v\|^p_{2,\infty} \leq \epsilon_0$. Assume that all representations $\gama_i$ satisfy the N.V.S. property for the respective dictionaries $\D_i$, and that $\|\gama_i\|^s_{0,\infty} = \lambda_i < \frac{1}{2}\left(1+\frac{1}{\mu(\D_i)}\right)$, for $1\leq i \leq L$ and $\|\gama_L\|^s_{0,\infty} = \lambda_L \leq \frac{1}{3}\left(1+\frac{1}{\mu(\D^{(L)})}\right)$. Let
	\begin{equation}
		\hat{\gama}_L = \underset{\gama}{\arg\min} \| \y - \D^{(L)} \gama \||^2_2 + \zeta_L \|\gama\|_1,
	\end{equation}
	for $\zeta_L = 4\epsilon_0$, and set $\hat{\gama}_{i-1} = \D_i\hat{\gama}_i$, $i=L,\dots,1$.  Then,
	\begin{enumerate}
		\item $Supp(\hat{\gama}_i) \subseteq Supp(\gama_i)$,
		\item $\|\hat{\gama}_i - \gama_i\|^p_{2,\infty} \leq \epsilon_L  \displaystyle\prod\limits_{j=i+1}^{L} \sqrt{\frac{3 c_j}{2}}$,
	\end{enumerate}
	hold for every layer $1\leq i\leq L$, where $\epsilon_L = \frac{15}{2}\ \epsilon_0 \sqrt{\|\gama_L\|^p_{0,\infty}}$ is the error at the last layer, and  $c_j$ depends on the ratio between the local dimensions of the layers, $c_j = \Bigl\lceil \frac{2n_{j-1}-1}{n_j} \Bigr\rceil$.
\end{thm}

\begin{thm}{Stable recovery of the Multi-Layer Pursuit Algorithm in the greedy case:} \label{Thm:StabilityPursuitOMP} \\
	Suppose a signal $\x(\gama_i) \in \M_\lamda$ is contaminated with energy-bounded noise $\v$, such that $\y = \x + \v$, $\|\y-\x\|_2 \leq \mathcal{E}_0$, and $\epsilon_0 = \|\v\|^\pp_{2,\infty}$. Assume that all representations $\gama_i$ satisfy the N.V.S. property for the respective dictionaries $\D_i$, with $\|\gama_i\|^s_{0,\infty} = \lambda_i < \frac{1}{2}\left(1+\frac{1}{\mu(\D_i)}\right)$, for $1\leq i \leq L$, and
	\begin{equation} \label{omp_hypothesis}
	\|\gama_L\|^s_{0,\infty} < \frac{1}{2}\left( 1+\frac{1}{\mu(\D^{(L)})} \right)-\frac{1}{\mu(\D^{(L)})}\cdot\frac{\epsilon_0}{|\gamma_{L}^{min}|},
	\end{equation}
	where $\gamma_{L}^{min}$ is the minimal entry in the support of $\gama_{L}$.
	Consider approximating the solution to the Pursuit step in Algorithm \ref{Alg:MulilayerPursuit} by running Orthogonal Matching Pursuit for $\|\gama_L\|_0$ iterations. Then, for every $i^{th}$ layer,
	\begin{enumerate}
		\item $Supp(\hat{\gama}_i) \subseteq Supp(\gama_i)$,
		\item $\|\hat{\gama}_i - \gama_i\|^2_2 \leq \frac{\mathcal{E}_0^2}{1-\mu(\D^{(L)})(\|\gama_L\|^s_{0,\infty}-1)} \left(\frac{3}{2}\right)^{L-i}$.
	\end{enumerate}
\end{thm}

The proofs of both Theorems \ref{Thm:StabilityPursuitLasso} and \ref{Thm:StabilityPursuitOMP} are included in the Supplementary Material \ref{app:GuaranteesConvexRelaxation} and \ref{app:StableGuaranteesGreedy}, respectively. The coefficient $c_j$ refers to the ratio between the filter dimensions at consecutive layers, and assuming $n_i \approx n_{i+1}$ (which indeed happens in practice), this coefficient is roughly 2. 
Importantly, and unlike the bounds provided for the layer-wise pursuit algorithm, the recovery guarantees are the tightest for the inner-most layer, and the bound increases slightly towards shallower representations. The relaxation to the $\ell_1$ norm, in the case of the BP formulation, provides local error bounds, while the guarantees for the greedy version, in its OMP implementation, yield a global alternative.

Before proceeding, one might wonder if the above conditions imposed on the representations and dictionaries are too severe and whether the set of signals satisfying these is empty. This is, in fact, not the case. As shown in \cite{Papyan2016convolutional}, multi-layer convolutional dictionaries can be constructed by means of certain wavelet functions, effectively achieving mutual coherence values in the order of $10^{-3}$, leaving ample room for sampling sparse representations satisfying the theorems' assumptions. On the other hand, imposing a constraint on the number of non-zeros in a representation $\gama_{i-1} = \D_i \gama_i$ implies that part of the support of the atoms in $\D_i$ will be required to overlap. The N.V.S. property simply guarantees that whenever these overlaps occur, they will not cancel each other. Indeed, this happens with probability 1 if the non-zero coefficients are drawn from a Normal distribution. We further comment and exemplify this in the Supplementary Material \ref{app:DiscussionNVS}.


\subsection{Projecting General Signals}

In the most general case, i.e. removing the assumption that $\y$ is close enough to a signal in the model, Algorithm \ref{Alg:MulilayerPursuit} by itself might not solve $\PM$. Consider we are given a general signal $\y$ and a model $\M_\lamda$, and we run the ML-CSC Pursuit with $k = \lambda_L$ obtaining a set of representations $\{\hat{\gama}_{j}\}$. Clearly $\|\hat{\gama}_L\|^s_{0,\infty} \leq \lambda_L$. Yet, nothing guarantees that $\| \hat{\gama}_{i}\|^s_{0,\infty} \leq \lambda_i$ for $i<L$. In other words, in order to solve $\PM$ one must guarantee that all sparsity constraints are satisfied. 

Algorithm \ref{Alg:ProjectionAlgorithm} progressively recovers sparse representations to provide a projection for any general signal $\y$. The solution is initialized with the zero vector, and then the OMP algorithm is applied with a progressively larger $\ell_{0,\infty}$ constraint on the deepest representation\footnote{Instead of repeating the pursuit from scratch at every iteration, one might-warm start the OMP algorithm by employing current estimate, $\hat{\gama}_L$, as initial condition so that only new non-zeros are added.}, from 1 to $\lambda_L$. The only modification required to run the OMP in this setting is to check at every iteration the value of $\|\hat{\gama}_L\|^s_{0,\infty}$, and to stop accordingly. At each step, given the estimated $\hat{\gama}_L$, the intermediate features and their $\ell_{0,\infty}$ norms, are computed. If all sparsity constraints are satisfied, then the algorithm proceeds. If, on the other hand, any of the constraints is violated, the previously computed $\x^\ast$ is reported as the solution. Note that this algorithm can be improved: if a constraint is violated, one might consider back-tracking the obtained deepest estimate and replacing the last obtained non-zero by an alternative solution, which might allow for the intermediate constraints to be satisfied. For simplicity, we present the completely greedy approach as in Algorithm \ref{Alg:ProjectionAlgorithm}.


\begin{algorithm}[t] 
	Init: $\x^\ast = \mathbf{0}$ \; 
	\For{$k = 1 : \lambda_L $}{ 
	$ \hat{\gama}_L \leftarrow \text{OMP}(\y,\D^{(L)},k)$ \;
	\For{$j = L : -1 : 1$}{
		$\hat{\gama}_{j-1} \leftarrow \D_j\hat{\gama}_j $\;
		}
		\If{$ \|\hat{\gama}_i\|^s_{0,\infty} > \lambda_i$ for any $1\leq i < L$ }	
		{ break\;}
		\Else{$\x^\ast \leftarrow \D^{(i)}\hat{\gama}_i$\;}

	}
	
	\Return{$\x^\ast$}
	\caption{ML-CSC Projection Algorithm}
	\label{Alg:ProjectionAlgorithm}
\end{algorithm}
This algorithm can be shown to be a greedy approximation to an optimal algorithm, under certain assumptions, and we provide a sketch of the proof of this claim in the Supplementary Material \ref{app:SketchProofProjection}.
Clearly, while Algorithms \ref{Alg:MulilayerPursuit} and \ref{Alg:ProjectionAlgorithm} were presented separately, they are indeed related and one can certainly combine them into a single method. The distinction between them was motivated by making the derivations of our theoretical analysis and guarantees easier to grasp.
Nevertheless, stating further theoretical claims without the assumption of the signal $\y$ being close to an underlying $\x(\gama_i) \in \M_\lambda$ is non-trivial, and we defer a further analysis of this case for future work.

\subsection{Summary - Pursuit for the ML-CSC}

Let us briefly summarize what we have introduced so far. We have defined a projection problem, $\PM$, seeking for the closest signal in the model $\M_\lamda$ to the measurements $\y$. We have shown that if the measurements $\y$ are close enough to a signal in the model, i.e. $\y = \x(\gama_i) + \v$, with bounded noise $\v$, then the ML-CSC Pursuit in Algorithm \ref{Alg:MulilayerPursuit} manages to obtain approximate solutions that are not far from these representations, by deploying either the OMP or the BP algorithms. In particular, the support of the estimated sparse vectors is guaranteed to be a subset of the correct support, and so all $\hat{\gama}_i$ satisfy the model constraints. In doing so we have introduced the N.V.S. property, and we have proven that the CSC and ML-CSC models are locally stable. Lastly, if no prior information is known about the signal $\y$, we have proposed an OMP-inspired algorithm that finds the closest signal $\x(\gama_i)$ to any measurements $\y$ by gradually increasing the support of all representations $\hat{\gama}_i$ while guaranteeing that the model constraints are satisfied. 






\section{Learning the model}
\label{sec:Learning}
\label{sec:empty_models?}
The entire analysis presented so far relies on the assumption of the existence of proper dictionaries $\D_i$ allowing for corresponding \emph{nested sparse features} $\gama_i$. Clearly, the ability to obtain such representations greatly depends on the design and properties of these dictionaries. 

While in the traditional sparse modeling scenario certain analytically-defined dictionaries (such as the Discrete Cosine Transform) often perform well in practice, in the ML-CSC case it is hard to propose an off-the-shelf construction which would allow for any meaningful decompositions. To see this more clearly, consider obtaining $\hat{\gama}_L$ with Algorithm \ref{Alg:MulilayerPursuit} removing all other assumptions on the dictionaries $\D_i$. In this case, nothing will prevent $\hat{\gama}_{L-1} = \D_L\hat{\gama}_L$ from being dense. More generally, we have no guarantees that \emph{any} collection of dictionaries would allow for any signal with nested sparse components $\gama_i$. In other words, how do we know if the model represented by $\{\D_i\}_{i=1}^L$ is not empty?

To illustrate this important point, consider the case where $\D_i$ are random -- a popular construction in other sparsity-related applications. In this case, every atom from the dictionary $\D_L$ will be a random variable $\d_L^j \sim \mathcal{N}(\mathbf{0},\sigma^2_L\mathbf{I})$. In this case, one can indeed construct $\gama_L$, with $\|\gama_{L}\|^s_{0,\infty} \leq 2$, such that \emph{every entry} from $\gama_{L-1} = \D_L\gama_L$ will be a random variable $\gamma^j_{L-1}\sim \mathcal{N}(0,\sigma^2_L)$, $\forall\ j$. Thus, $\Pr\left(\gamma^j_{L-1}=0\right)=0$. As we see, there will not exist any sparse (or dense, for that matter) $\gama_L$ which will create a sparse $\gama_{L-1}$. In other words, for this choice of dictionaries, the ML-CSC model is empty.

\subsection{Sparse Dictionaries}

From the discussion above one can conclude that one of the key components of the ML-CSC model is sparse dictionaries: if both $\gama_L$ and $\gama_{L-1} = \D_L\gama_L$ are sparse, then atoms in $\D$ must indeed contain only a few non-zeros. We make this observation concrete in the following lemma.

\begin{lemma}{Dictionary Sparsity Condition}\label{lemma:SparseDictionaries} \\
	Consider the ML-CSC model $\M_\lamda$ described by the dictionaries $\{\D_i\}_{i=1}^L$ and the layer-wise $\ell_{0,\infty}$-sparsity levels $\lambda_1,\lambda_2,\dots,\lambda_L$. Given $\gama_L : \|\gama_L\|^s_{0,\infty} \leq \lambda_L$ and constants $c_i = \Bigl\lceil \frac{2n_{i-1}-1}{n_i} \Bigr\rceil$, the signal $\x = \D^{(L)} \gama_L \in \M_\lamda$ if
	\begin{equation}
		\|\D_i\|_0 \leq \frac{\lambda_{i-1}}{\lambda_i c_i}, \quad \forall\ 1<i\leq L.
	\end{equation}
\end{lemma}

The simple proof of this Lemma is included in the Supplementary Material \ref{app:SparseDictioanries}. Notably, while this claim does not tell us if a certain model is empty, it does guarantee that if the dictionaries satisfy a given sparsity constraint, one can simply sample from the model by drawing the inner-most representations such that $\|\gama_L \|^s_{0,\infty}\leq \lambda_L$. One question remains: how do we train such dictionaries from real data?

\subsection{Learning Formulation}
\label{subsec:Learning}

One can understand from the previous discussion that there is no hope in solving the $\PM$ problem for real signals without also addressing the learning of dictionaries $\D_i$ that would allow for the respective representations. To this end, considering the scenario where one is given a collection of $K$ training signals, $\{\y^k\}_{k=1}^K$, we upgrade the $\PM$ problem to a learning setting in the following way:
\begin{equation}
	\min_{ \{\gama^k_i\},\{\D_i\}} \quad \sum_{k=1}^K \| \y^k - \x^k(\gama^k_i,\D_i) \|^2_2 \quad \text{s.t.} \left\{
	\begin{array}{c}
		\x^k \in\mathcal{M}_\lamda, \\
		\|\d^j_i\|_2 = 1, \forall \ i,j.
	\end{array}
\right.		
\label{Eq:FirstLearningProblemm}
\end{equation}
We have included the constraint of every dictionary atom to have a unit norm to prevent arbitrarily small coefficients in the representations $\gama_i^k$. This formulation, while complete, is difficult to address directly: The constraints on the representations $\gama_i$ are coupled, just as in the pursuit problem discussed in the previous section. In addition, the sparse representations now also depend on the variables $\D_i$. In what follows, we provide a relaxation of this cost function that will result in a simple learning algorithm.

The problem above can also be understood from the perspective of minimizing the number of non-zeros in the representations at every layer, subject to an error threshold -- a typical reformulation of sparse coding problems. Our main observation arises from the fact that, since $\gama_{L-1}$ is function of both $\D_L$ and $\gama_L$, one can upper-bound the number of non-zeros in $\gama_{L-1}$ by that of $\gama_{L}$. More precisely, 
\begin{equation}
\|\gama_{L-1}\|^s_{0,\infty} \leq c_L \|\D_L\|_0\|\gama_{L}\|^s_{0,\infty},	
\end{equation}
where $c_L$ is a constant\footnote{From \cite{Papyan2016convolutional}, we have that $\|\gama_{L-1}\|^p_{0,\infty} \leq \|\D_L\|_0\|\gama_{L}\|^s_{0,\infty}$. From here, and denoting by $c_{L}$ the upper-bound on the number of patches in a stripe from $\gama_{L-1}$ given by $c_L = \Bigl\lceil \frac{2n_{L-1}-1}{n_L} \Bigr\rceil$, we can obtain a bound to $\|\gama_{L-1}\|^s_{0,\infty}$.}.
Therefore, instead of minimizing the number of non-zeros in $\gama_{L-1}$, we can address the minimization of its upper bound by minimizing both $\|\gama_{L}\|^s_{0,\infty}$ and $\|\D_L\|_0$. This argument can be extended to any layer, and we can generally write
 \begin{equation} 
	\|\gama_{i}\|^s_{0,\infty} \ \leq \ c \prod_{j=i+1}^L \|\D_j\|_0  \|\gama_L\|^s_{0,\infty}.
\end{equation}
In this way, minimizing the sparsity of any $i^{th}$ representation can be done implicitly by minimizing the sparsity of the last layer \emph{and} the number of non-zeros in the dictionaries from layer $(i+1)$ to $L$. Put differently, the sparsity of the intermediate convolutional dictionaries serve as proxies for the sparsity of the respective representation vectors. Following this observation, we now recast the problem in Equation \eqref{Eq:FirstLearningProblemm} into the following Multi-Layer Convolutional Dictionary Learning Problem: 
\begin{multline}
	\min_{\{\gama^k_L\},\{\D_i\}} \sum_{k=1}^K \| \y^k - \D_1\D_2\dots\D_L \gama^k_L \|^2_2 + \sum_{i=2}^L \zeta_i \|\D_i\|_0 \quad \\ \text{s.t.} \left\{
	\begin{array}{c}
		\|\gama^k_L\|^s_{0,\infty} \leq \lambda_L, \\
		\|\d^j_i\|_2 = 1, \forall \ i,j.
	\end{array}
\right.		
\label{Eq:SecondLearningProblemm}
\end{multline}
Under this formulation, this problem seeks for sparse representations $\gama^k_L$ for each example $\y^k$, while forcing the intermediate convolutional dictionaries (from layer 2 to $L$) to be sparse. The reconstructed signal, $\x = \D_1\gama_1$, is not expected to be sparse, and so there is no reason to enforce this property on $\D_1$. Note that there is now only one sparse coding process involved -- that of $\gama^k_L$ -- while the intermediate representations are never computed explicitly. Recalling the theoretical results from the previous section, this is in fact convenient as one only has to estimate the representation for which the recovery bound is the tightest.

Following the theoretical guarantees presented in Section \ref{sec:Projection}, one can alternatively replace the $\Loi$ constraint on the deepest representation by a convex $\ell_1$ alternative. The resulting formulation resembles the lasso formulation of the $\PM$ problem, for which we have presented theoretical guarantees in Theorem \ref{Thm:StabilityPursuitLasso}. In addition, we replace the constraint on the $\ell_2$ of the dictionary atoms by an appropriate penalty term, recasting the above problem into a simpler (unconstrained) form: 
\begin{multline}
	\min_{\{\gama^k_L\},\{\D_i\}} \sum_{k=1}^K \| \y^k - \D_1\D_2\dots\D_L \gama^k_L \|^2_2 +\\ \iota \sum_{i=1}^L \|\D_i\|^2_F + \sum_{i=2}^L \zeta_i \|\D_i\|_0 + \lambda \|\gama^k_L\|_1,
\label{Eq:ThirdLearningProblemm}
\end{multline}
where $\|\cdot\|_F$ denotes the Frobenius norm.
The problem in Equation \eqref{Eq:ThirdLearningProblemm} is highly non-convex, due to the $\ell_0$ terms and the product of the factors. In what follows, we present an online alternating minimization algorithm, based on stochastic gradient descent, which seeks for the deepest representation $\gama_L$ and then progressively updates the layer-wise convolutional dictionaries.

For each incoming sample $\y^k$ (or potentially, a mini-batch), we will first seek for its deepest representation $\gama^k_L$ considering the dictionaries fixed. This is nothing but the $\PM$ problem in \eqref{Eq:PMproblem}, which was analyzed in detail in the previous sections, and its solution will be approximated through iterative shrinkage algorithms. Also, one should keep in mind that while representing each dictionary by $\D_i$ is convenient in terms of notation, these matrices are never computed explicitly -- which would be prohibitive. Instead, these dictionaries (or their transpose) are applied effectively through convolution operators. In turn, this implies that images are not vectorized but processed as 2 dimensional matrices (or 3-dimensional tensors for multi-channel images). In addition, these operators are very efficient due to their high sparsity, and one could in principle benefit from specific libraries to boost performance in this case, such as the one in \cite{liu2015sparse}.

Given the obtained $\gama^k_L$, we then seek to update the respective dictionaries. As it is posed -- with a global $\ell_0$ norm over each dictionary -- this is nothing but a generalized pursuit as well. Therefore, for each dictionary $\D_i$, we minimize the function in Problem \eqref{Eq:ThirdLearningProblemm} by applying $T$ iterations of projected gradient descent. 
This is done by computing the gradient of the $\ell_2$ terms in Problem \eqref{Eq:ThirdLearningProblemm} (call it $f(\D_i)$) with respect to a each dictionary $\D_i$ (i.e., $\nabla f(\D_i)$), making a gradient step and then applying a hard-thresholding operation, $\mathcal{H}_{\zeta_i}(\cdot)$, depending on the parameter $\zeta_i$. This is simply an instance of the Iterative Hard Thresholding algorithm \cite{Blumensath2008}.
In addition, the computation of $\nabla f(\D_i)$ involves only multiplications the convolutional dictionaries for the different layers. The overall algorithm is depicted in Algorithm \ref{Alg:ML-CDL}, and we will expand on further implementation details in the results section.

\begin{algorithm}
\setstretch{1.1}
	\KwData{Training samples $\{\y_k\}_{k=1}^K$, initial convolutional dictionaries $\{\D_i\}_{i=1}^L$}

	\For{$k = 1,\dots,K$}{
		Draw $\y_k$ at random\;
		Sparse Coding: $\gama_L \leftarrow \underset{\gama}{\arg\min}\  \|\y_k - \D^{(L)}\gama\|_2 + \lambda \|\gama\|_{1} $ \;
		Update Dictonaries: \\
		\For{$i = L,\dots,2$}{
			\For{t = 1,\dots,T}{
			$\D_i^{t+1} \leftarrow \mathcal{H}_{\zeta_i}\left[ \D^t_i - \eta  \nabla f(\D^t_i) \right]$ \;
			}
		}
		\For{t = 1,\dots,T}{
			$\D_1^{t+1} \leftarrow \D^t_1 - \eta  \nabla f(\D^t_1)$ \;
			} 
		}
	\caption{Multi-Layer Convolutional Dictionary Learning}
	\label{Alg:ML-CDL}
\end{algorithm}


\begin{figure*}
\begin{center}
		
	\begin{minipage}{.22\textwidth}
		a) \\[.3cm]
		\includegraphics[width = .78\textwidth]{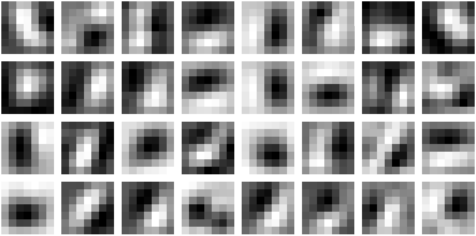}
	\end{minipage} 
		\begin{minipage}{.72\textwidth}
		b) \\[.3cm]
		\includegraphics[width = \textwidth]{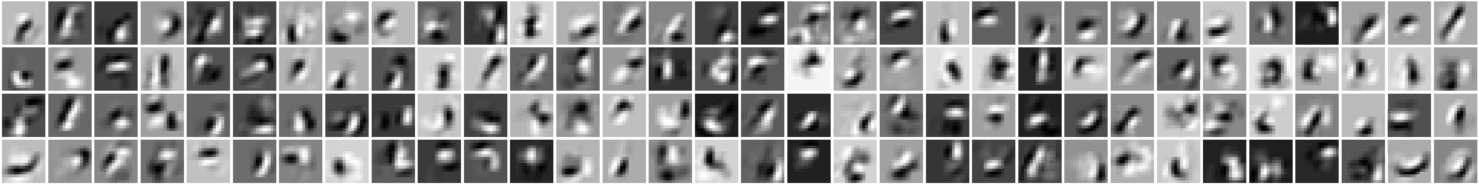}
	\end{minipage}\\[.2cm]

	\begin{minipage}{.95\textwidth}
		c) \\[.3cm]
		\includegraphics[width = \textwidth]{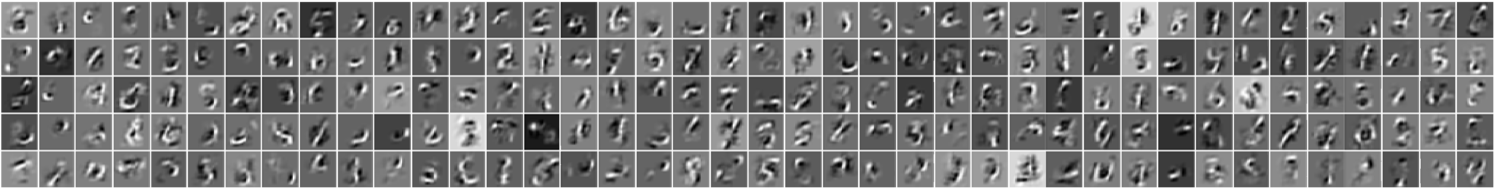}
	\end{minipage}
\caption{ML-CSC model trained on the MNIST dataset. a) The local filters of the dictionary $\D_1$. b) The local filters of the effective dictionary $\D^{(2)} = \D_1\D_2$. c) Some of the 1024 local atoms of the effective dictionary $\D^{(3)}$ which, because of the dimensions of the filters and the strides, are global atoms of size $28\times28$.}
\label{fig:MnistModel}
\end{center}
\end{figure*}

The parameters of the models involve the $\ell_1$ penalty of the deepest representation, i.e. $\lambda$, and the parameter for each dictionary, $\zeta_i$. The first parameter can be set manually or determined so as to obtain a given given representation error. On the other hand, the dictionary-wise $\zeta_i$ parameters are less intuitive to establish, and the question of how to set these values for a given learning scenario remains a subject of current research. Nevertheless, we will show in the experimental section that setting these manually results in effective constructions.

Note this approach can also be employed to minimize Problem \eqref{Eq:SecondLearningProblemm} by introducing minor modifications: In the sparse coding stage, the Lasso is replaced by a $\ell_{0,\infty}$ pursuit, which can be tackled with a greedy alternative as the OMP (as described in Theorem \ref{Thm:StabilityPursuitOMP}) or by an Iterative Hard Thresholding alternative \cite{Blumensath2008}. In addition, one could consider employing the $\ell_1$ norm as a surrogate for the $\ell_0$ penalty imposed on the dictionaries. In this case, their update can still be performed by the same projected gradient descent approach, though replacing the hard thresholding with its soft counterpart.

Before moving on, and even though an exhaustive computational complexity analysis is out of the scope of this paper, we want to briefly comment on the general aspects of the algorithm's complexity. For a particular network architecture (number of layers, number of filters per layer, filter sizes, etc) let us denote by $\mathcal{C}$ the complexity of applying the forward pass -- or in other words, multiplying by $\D^{(L)^T}$ -- on an input image, or a minibach (i.e., for each $k^{th}$ iteration). The sparse coding step in our algorithm is carried with iterative shrinkage methods, and assuming these algorithms are run for $\tau$ iterations, the complexity incurred in each sparse coding step is\footnote{Each such iteration actually involves the application of a forward and backward pass, resulting from the fact that one needs to apply $\D^{(L)}$ and $\D^{(L)^T}$.} $\mathcal{O}(\tau\mathcal{C})$. The update of the dictionaries, on the other hand, requires computing the gradient for each set of filters. Each of these gradients involves, roughly speaking, the computation of yet another forward and backward pass\footnote{The dictionary gradients can actually be computed more efficiently if intermediate computations are saved (and stored), incurring in $\mathcal{O}(L\log_2(L))$ convolution operators. Thus, in this case the dictionary update stage is $\mathcal{O}(\log_2(L)T \mathcal{C})$. We defer the implementation of this more efficient algorithm for future work.}. In this way, the dictionary update stage is $\mathcal{O}(LT \mathcal{C})$.
Note that we are disregarding the shrinkage operators both on the representations and on the filters, which are entry-wise operations that are negligible when compared to applying $\D^{(L)}$ or its transpose. As can be seen, the complexity of our algorithm is approximately $(\tau+TL)$ times that of a similar CNNs architectures. Finally, note that we are not considering the (important) fact that, in our case, the convolutional kernels are sparse, and as such they may incur in significantly cheaper computations. This precise analysis, and how to maximize the related computational benefit, is left for future work.

\subsection{Connection to related works}

Naturally, the proposed algorithm has tight connections to several recent dictionary learning approaches. For instance, our learning formulation is closely related to the Chasing Butterflies approach in \cite{Lemagoarou15}, and our resulting algorithm can be interpreted as a particular case of the FAUST method, proposed in the inspiring work from \cite{le2016flexible}. FAUST decomposes linear operators into sparse factors in a hierarchical way in the framework of a batch learning algorithm, resulting in improved complexity. Unlike that work, our multi-layer decompositions are not only sparse but also convolutional, and they are updated within a stochastic optimization framework. 
The work in \cite{Chabiron2013}, on the other hand, proposed a learning approach where the dictionary is expressed as a cascade of convolutional filters with sparse kernels, and they effectively showed how this approach can be used to approximate large-dimensional analytic atoms such as those from wavelets and curvelets. 
As our proposed approach effectively learns a sparse dictionary, our work also shares similarities with the double-sparsity work from \cite{Rubinstein2010}. In particular, in its Trainlets version \cite{Sulam2016}, the authors proposed to learn a dictionary as a sparse combination of cropped wavelets atoms. From the previous comment on the work from \cite{Chabiron2013}, this could also potentially be expressed as a product of sparse convolutional atoms. All these works, as well as our approach, essentially enforce extra regularization into the dictionary learning problem. As a result, these methods perform better in cases with corrupted measurements, in high dimensional settings, and in cases with limited amount of training data (see \cite{Rubinstein2010,Sulam2016}). 

What is the connection between this learning formulation and that of deep convolutional networks? Recalling the analysis presented in \cite{Papyan2016convolutional}, the Forward Pass is nothing but a layered non-negative thresholding algorithm, the simplest form of a pursuit for the ML-CSC model with layer-wise deviations. Therefore, if the pursuit for $\hat{\gama}_L$ in our setting is solved with such an algorithm, then the problem in \eqref{Eq:ThirdLearningProblemm} \emph{implements a convolutional neural network with only one RELU operator at the last layer, with sparse-enforcing penalties on the filters.} Moreover, due the data-fidelity term in our formulation, the proposed optimization problem provides nothing but a convolutional sparse autoencoder. As such, our work is related to the extensive literature in this topic. For instance, in \cite{ng2011sparse}, sparsity is enforced in the hidden activation layer by employing a penalty term proportional to the KL divergence between the hidden unit marginals and a target sparsity probability. 

Other related works include the k-sparse autoencoders \cite{makhzani2013k}, where the hidden layer is constrained to having exactly $k$ non-zeros. In practice, this boils down to a hard thresholding step of the hidden activation, and the weights are updated with gradient descent. In this respect, our work can be thought of a generalization of this work, where the pursuit algorithm is more sophisticated than a simple thresholding operation, and where the filters are composed by a cascade of sparse convolutional filters. 
More recently, the work in \cite{makhzani2015winner} proposed the \emph{winner-take-all} autoencoders. In a nutshell, these are non-symmetric autoencoders having a few convolutional layers (with ReLu non-linearities) as the encoder, and a simple linear decoder. Sparsity is enforced in what the authors refer to as ``spatial'' and a ``lifetime'' sparsity. 

Finally, and due to the fact that our formulation effectively provides a convolutional network with sparse kernels, our approach is reminiscent of works attempting to sparsify the filters in deep learning models. For instance, the work in \cite{liu2015sparse} showed that the weights of learned deep convolutional networks can be sparsified without considerable degradation of classification accuracy. Nevertheless, one should perpend the fact that these works are motivated merely by cheaper and faster implementations, whereas our model is intrinsically built by theoretically justified sparse kernels. We do not attempt to compare our approach to such sparsifying methods at this stage, and we defer this to future work.

In light of all these previous works, the practical contribution of the learning algorithm presented here is to demonstrate, as we will see in the following Experiments section, that our online block-coordinate descent method can be effectively deployed in an unsupervised setting competing favorably with state of the art dictionary learning and convolutional network auto-encoders approaches.

\section{Experiments}
\label{sec:experiments}
We now provide experimental results to demonstrate several aspects of the ML-CSC model. As a case-study, we consider the MNIST dataset \cite{MNIST}. We define our model as consisting of 3 convolutional layers: the first one contains 32 local filters of size $7\times 7$ (with a stride of 2), the second one consists of 128 filters of dimensions $5\times5\times32$ (with a stride of 1), and the last one contains 1024 filters of dimensions $7\times7\times128$. At the third layer, the effective size of the atoms is 28 -- representing an entire digit. 

Training is performed with Algorithm \ref{Alg:ML-CDL}, using a mini-batch of 100 samples per iteration. For the Sparse Coding stage, we leverage an efficient implementation of FISTA \cite{beck2009fast}, and we adjust the penalty parameter $\lambda$ to obtain roughly 15 non-zeros in the deepest representation $\gama_3$. The $\zeta_i$ parameters, the penalty parameters for the dictionaries sparsity levels, are set manually for simplicity. In addition, and as it is commonly done in various Gradient Descent methods, we employ a momentum term for the update of the dictionaries $\D_i$ within the projected gradient descent step in Algorithm \ref{Alg:ML-CDL}, and set its memory parameter to 0.9. The step size is set to 1, the update dictionary iterations is set as $T=1$, $\iota = 0.001$, and we run the algorithm for 20 epochs, which takes approximately 30 minutes. Our implementation uses the Matconvnet library, which leverages efficient functions for GPU\footnote{All experiments are run on a 16 i7 cores Windows station with a NVIDIA GTX 1080 Ti.}. No pre-processing was performed, with the exception of the subtraction of the mean image (computed on the training set).

We depict the evolution of the Loss function during training in Figure \ref{fig:Training}, as well as the sparsity of the second and third dictionaries (i.e., 1 minus the number of non-zero coefficients in the filters relative to the filters dimension) and the average residual norm. The resulting model is depicted in Figure \ref{fig:MnistModel}. One can see how the first layer is composed of very simple small-dimensional edges or blobs. The second dictionary, $\D_2$, is effectively $99\%$ sparse, and its non-zeros combine a few atoms from $\D_1$ in order to create slightly more complex edges, as the ones in the effective dictionary $\D^{(2)}$. Lastly, $\D_3$ is $99.8\%$ sparse, and it combines atoms from $\D^{(2)}$ in order to provide atoms that resemble different kinds (or parts) of digits. These final global atoms are nothing but a linear combination of local small edges by means of convolutional sparse kernels. 


\begin{figure}
\begin{center}
		\includegraphics[trim = 20 0 20 0, width = .49\textwidth]{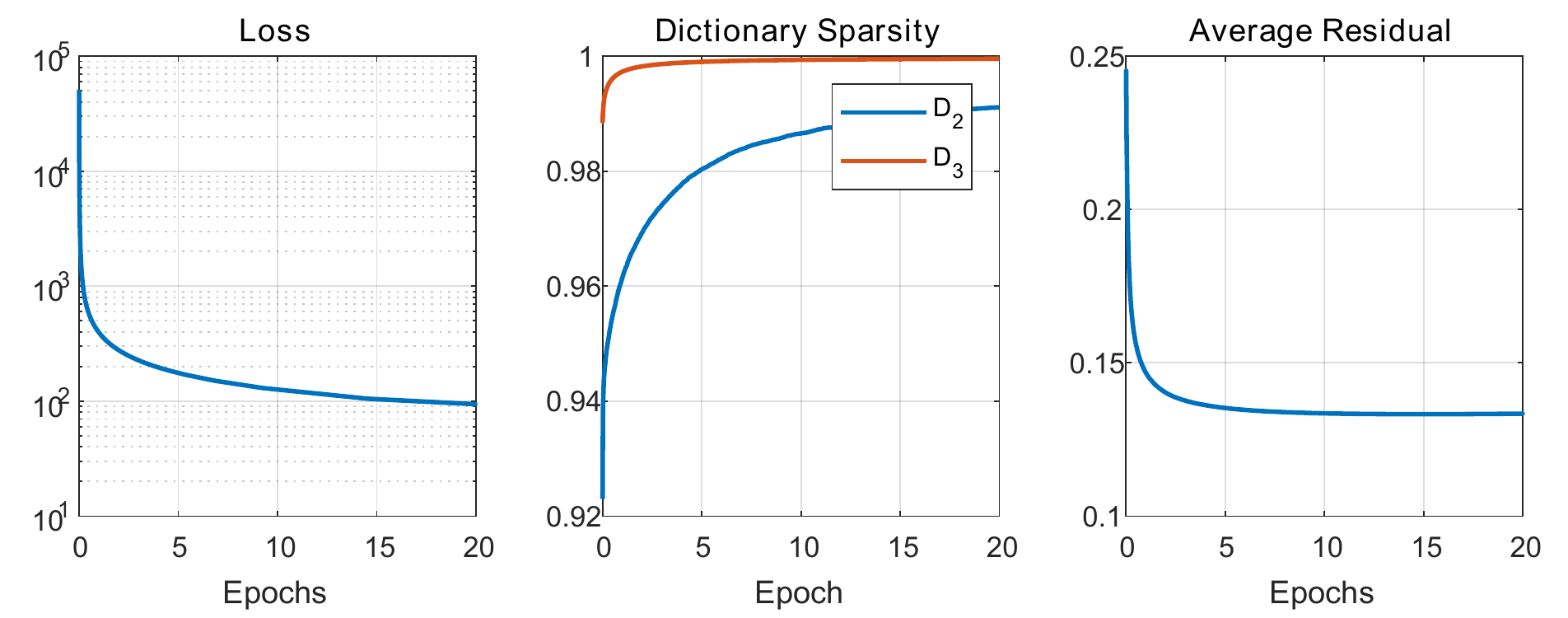}
\caption{Evolution of the Loss function, sparsity of the convolutional dictionaries and average residual norm during training on the MNIST dataset.}
\label{fig:Training}
\end{center}
\end{figure}

Interestingly, we have observed that the mutual coherence of the effective dictionaries do not necessarily increase with the layers, and they often decrease with the depth.
While this measure relates to worst-case analysis conditions and do not mean much in the context of practical performance, one can see that the effective dictionary indeed becomes less correlated as the depth increases. This is intuitive, as very simple edges -- and at every location -- are expected to show large inner products, larger than the correlation of two more complex number-like structures. This effect can be partially explained by the dictionary redundancy: having 32 local filters in $\D_1$ (even while using a stride of 2) implies a 8-fold redundancy in the effective dictionary at this level. This redundancy decreases with the depth (at this least for the current construction), and at the third layer one has \emph{merely} 1024 atoms (redundancy of about 1.3, since the signal dimension is $28^2$).

\begin{figure}
	\begin{center}
		\includegraphics[width = .49\textwidth]{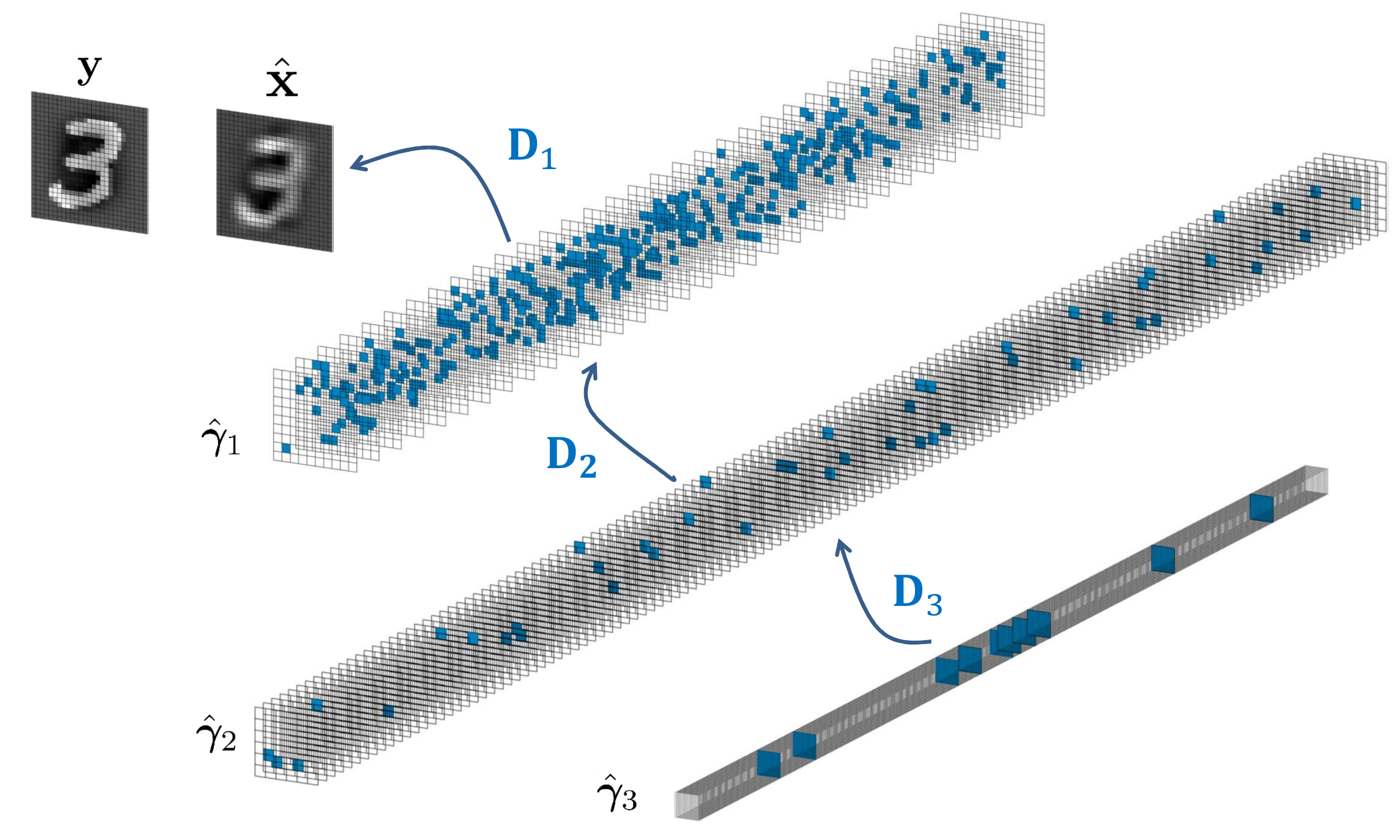}
		\caption{Decompositions of an image from MNIST in terms of its nested sparse features $\gama_i$ and multi-layer convolutional dictionaries $\D_i$.}
		\label{fig:NumberDecomposition}
	\end{center}
\end{figure}
We can also find the multi-layer representation for real images -- essentially solving the projection problem $\PM$. In Figure \ref{fig:NumberDecomposition}, we depict the multi-layer features $\gama_i$, $i = 1,2,3$, obtained with the Algorithm \ref{Alg:MulilayerPursuit}, that approximate an image $\y$ (not included in the training set). Note that all the representations are notably sparse thanks to the very high sparsity of the dictionaries $\D_2$ and $\D_3$. These decompositions (any of them) provide a sparse decomposition of the number 3 at different scales, resulting in an approximation $\hat{\x}$. Naturally, the quality of the approximation can be improved by increasing the cardinality of the representations.


\subsection{Sparse Recovery}

The first experiment we explore is that of recovering sparse vectors from corrupted measurements, in which we will compare the presented ML-CSC Pursuit with the Layered approach from \cite{Papyan2016convolutional}. For the sake of completion and understanding, we will first carry this experiment in a synthetic setting and then on projected real digits, leveraging the dictionaries obtained in the beginning of this section.

We begin by constructing a 3 layers ``non-convolutional'' \footnote{The non-convolutional case is still a ML-CSC model, in which the signal dimension is the same as the length of the atoms $n$, and with a stride of the same magnitude $n$. We choose this setting for the synthetic experiment to somewhat favor the results of the layered pursuit approach.} model for signals of length $200$, with the dictionaries having 250, 300, and 350 atoms, respectively. The first dictionary is constructed as a random matrix, whereas the remaining ones are composed of sparse atoms with random supports and a sparsity of $99\%$. Finally, 500 representations are sampled by drawing sparse vectors $\gama_L$, with a target sample sparsity $k$ and normally distributed coefficients. We generate the signals as $\x = \D^{(i)}\gama_i$, and then corrupt them with Gaussian noise ($\sigma = 0.02$) obtaining the measurements $\y = \x(\gama_i) + \v$. 

\begin{figure}
	\centering
	\begin{subfigure}{.5\textwidth}
		\includegraphics[trim = 30 40 30 0, width = \textwidth]{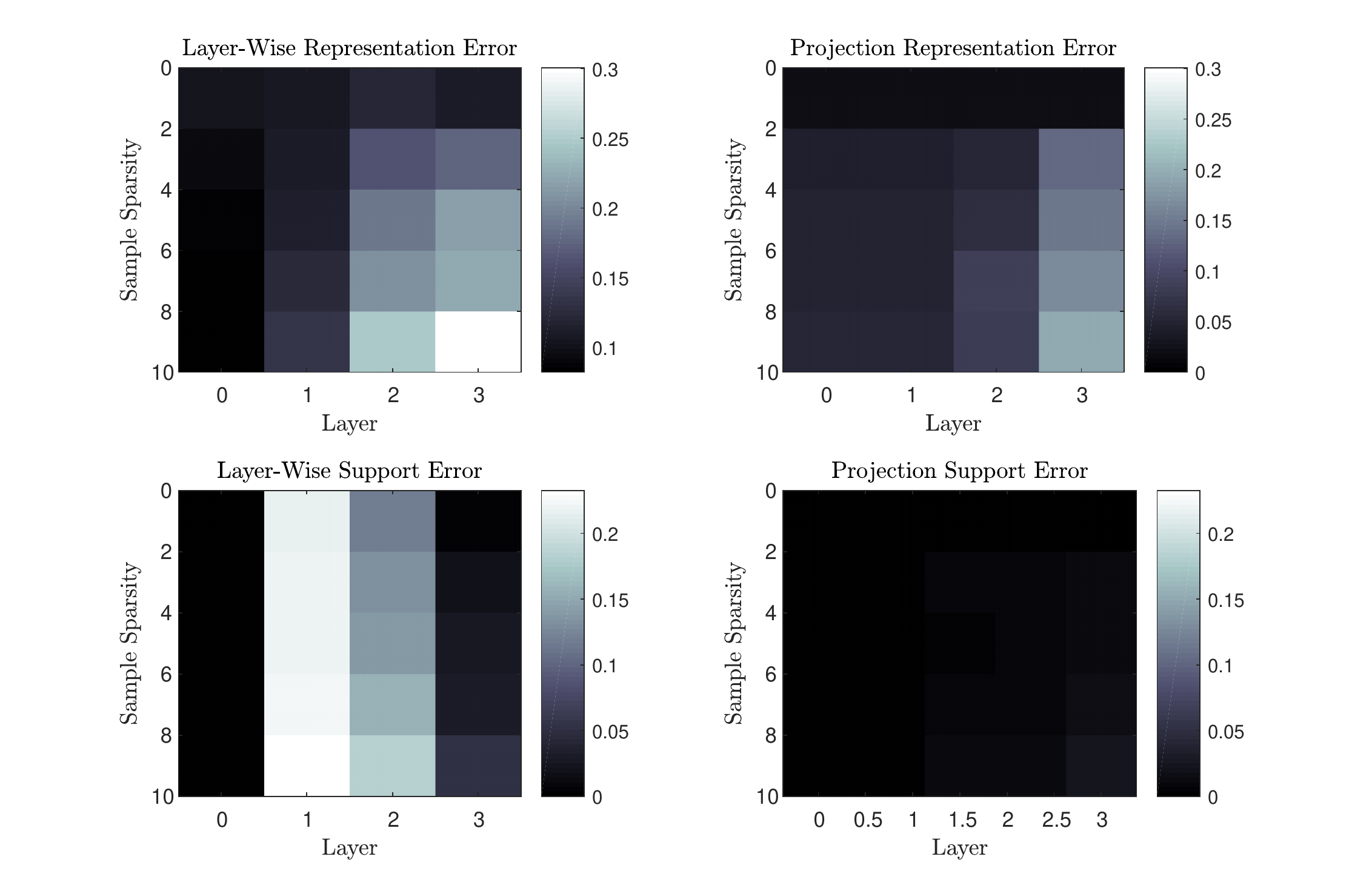} \\
		\caption{\footnotesize Synthetic signals.}
		\label{fig:SyntheticExperiment}
	\end{subfigure}
		

	\begin{subfigure}{.5\textwidth}
		\includegraphics[trim = 30 40 30 -20, width = \textwidth]{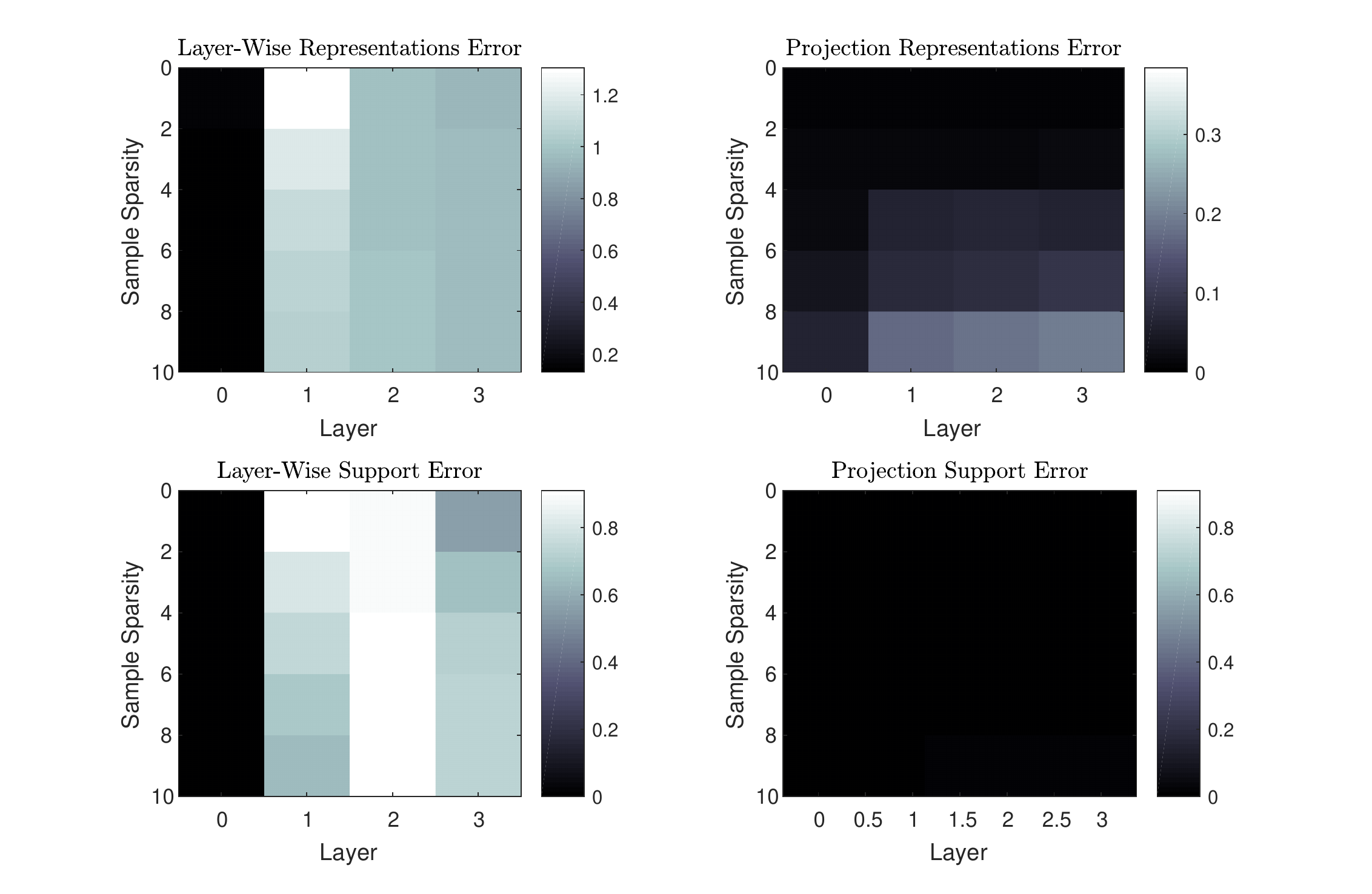} \\
		\caption{\footnotesize MNIST signals.}
		\label{fig:Recovery_MnistExperiment}
	\end{subfigure}
	\caption{Recovery of representations from noisy MNIST digits. Top: normalized $\ell_2$ error between the estimated and the true representations. Bottom: normalized intersection between the estimated and the true support of the representations.}
\end{figure}

In order to evaluate our projection approach, we run Algorithm \ref{Alg:MulilayerPursuit} employing the Subspace Pursuit algorithm \cite{dai2009subspace} for the sparse coding step, with the oracle target cardinality $k$. Recall that once the deepest representations $\hat{\gama}_L$ have been obtained, the inner ones are simply computed as $\hat{\gama}_{i-1} = \D_i\hat{\gama}_i$. In the layered approach from \cite{Papyan2016convolutional}, on the other hand, the pursuit of the representations progresses sequentially: first running a pursuit for $\hat{\gama}_1$, then employing this estimate to run another pursuit for $\hat{\gama}_2$, etc. In the same spirit, we employ Subspace Pursuit layer by layer, employing the oracle cardinality of the representation at each stage. The results are presented in Figure \ref{fig:SyntheticExperiment}: at the top we depict the relative $\ell_2$ error of the recovered representations ($\|\hat{\gama}_i- {\gama}_i\|_2 / \| {\gama}_i\|_2$) and, at the bottom, the normalized intersection of the supports \cite{Elad_Book}, both as a function of the sample cardinality $k$ and the layer depth. 

The projection algorithm manages to retrieve the representations $\hat{\gama}_i$ more accurately than the layered pursuit, as evidenced by the $\ell_2$ error and the support recovery. The main reason behind the difficulty of the layer-by-layer approach is that the entire process relies on the correct recovery of the first layer representations, $\hat{\gama}_1$. If these are not properly estimated (as evidenced by the bottom-left graph), there is little hope for the recovery of the deeper ones. In addition, these representations $\gama_1$ are the least sparse ones, and so they are expected to be the most challenging ones to recover. The projection alternative, on the other hand, relies on the estimation of the deepest $\hat{\gama}_L$, which are very sparse. Once these are estimated, the remaining ones are simply computed by propagating them to the shallower layers. Following our analysis in the Section \ref{sec:StabilityPursuits}, if the support of $\hat{\gama}_L$ is estimated correctly, so will be the support of the remaining representations $\hat{\gama}_i$.

We now turn to deploy the 3 layer convolutional dictionaries for real digits obtained previously. To this end we take 500 test digits from the MNIST dataset and project them on the trained model, essentially running Algorithm \ref{Alg:MulilayerPursuit} and obtaining the representations $\gama_i$. We then create the noisy measurements as $\y = \D^{(i)}\gama_i + \v$, where $\v$ is Gaussian noise with $\sigma = 0.02$. We then repeat both pursuit approaches to estimate the underlying representations, obtaining the results reported in Figure \ref{fig:Recovery_MnistExperiment}. 

Clearly, this represents a significantly more challenging scenario for the layered approach, which recovers only a small fraction of the correct support of the sparse vectors. The projection algorithm, on the other hand, provides accurate estimations with negligible mistakes in the estimated supports, and very low $\ell_2$ error. Note that the $\ell_2$ error has little significance for the Layered approach, as this algorithm does not manage to find the true supports.
The reason for the significant deterioration in the performance of the Layered algorithm is that this method actually finds alternative representations $\hat{\gama}_1$, of the same sparsity, providing a lower fidelity term than the projection counterpart for the first layer. However, these estimates $\hat{\gama}_1$ do not necessarily provide a signal in the model, which causes further errors when estimating $\hat{\gama}_2$. 

\subsection{Sparse Approximation}

A straight forward application for unsupervised learned model is that of approximation: how well can one approximate or reconstruct a signal given only a few $k$ non-zero values from some representation? In this subsection, we study the performance of the ML-CSC model for this task while comparing with related methods, and we present the results in Figure \ref{fig:MtermApp}. The model is trained on $60K$ training examples, and the M-term approximation is measured on the remaining $10K$ testing samples. All of the models are designed with 1K hidden units (or atoms).

Given the close connection of the ML-CSC model to sparse auto-encoders, we present the results obtained by approximating the signals with sparse autoencoders \cite{ng2011sparse} and k-sparse autoencoders \cite{makhzani2013k}. In particular, the work in \cite{ng2011sparse} trains sparse auto-encoders by penalizing the KL divergence between the activation distribution of the hidden neurons and that of a binomial distribution with a certain target activation rate. As such, the resulting activations are never truly sparse. For this reason, since the M-term approximation is computed by picking the highest entries in the hidden neurons and setting the remaining ones to zero, this method exhibits a considerable representation error.

\begin{figure}
	\begin{center}
		\includegraphics[trim = 10 0 0 10, width = .49\textwidth]{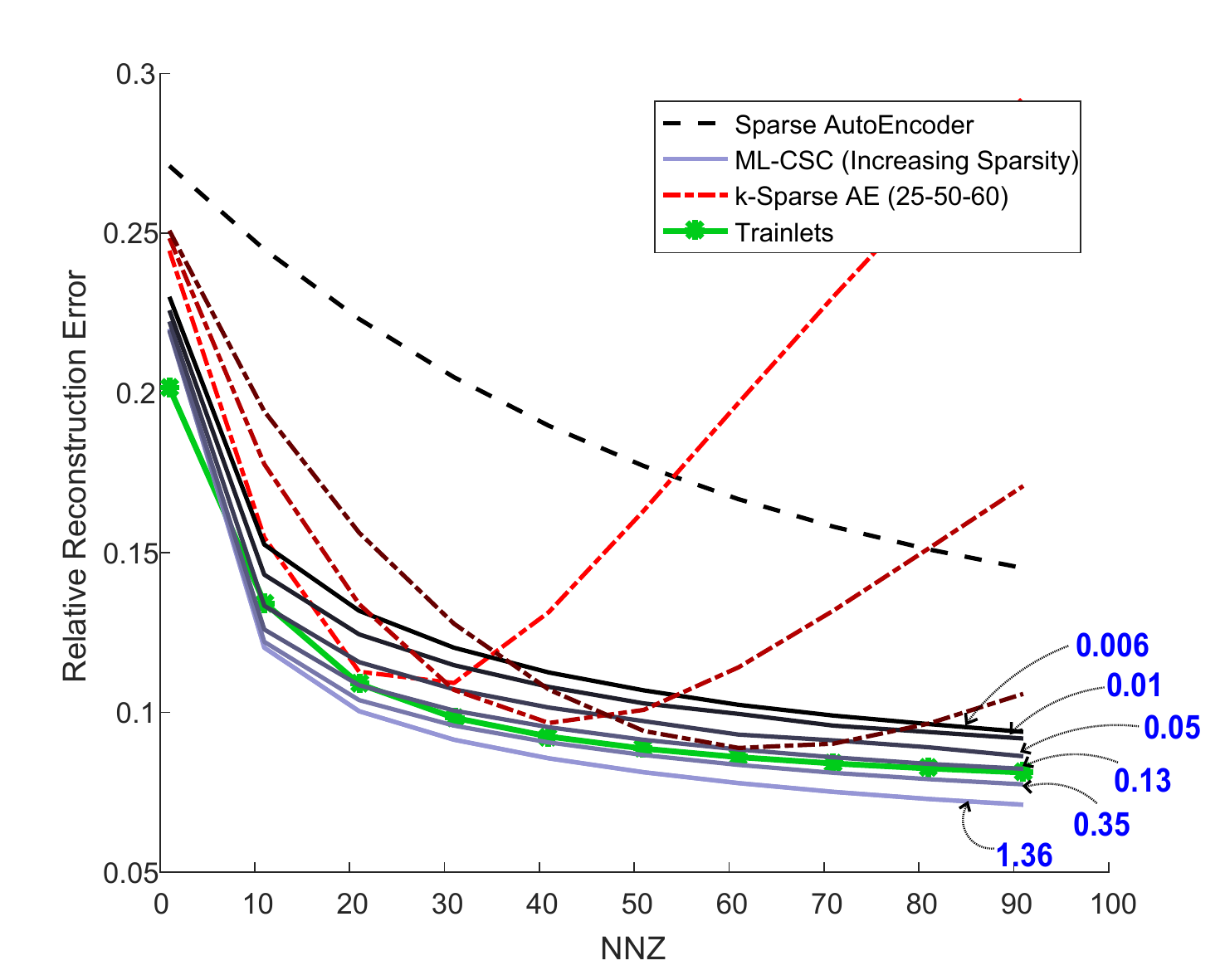}
		\caption{M-term approximation for MNIST digits, comparing sparse autoencoders \cite{ng2011sparse}, k-sparse autoencoders \cite{makhzani2013k}, trainlets (OSDL) \cite{Sulam2016}, and the proposed ML-CSC for models with different filter sparsity levels. The relative number of parameters is depicted in blue. }
		\label{fig:MtermApp}
	\end{center}
\end{figure}

K-sparse auto-encoders perform significantly better, though they are sensitive to the number of non-zeros used during training. Indeed, if the model is trained with 25 non-zeros per sample, the model performs well for a similar range of cardinalities. Despite this sensitivity on training, their performance is remarkable considering the simplicity of the pursuit involved: the reconstruction is done by computing $\hat{\x} = \W\hat{\gama}_k+\b'$, where $\hat{\gama}_k$ is a k-sparse activation (or feature) obtained by hard thresholding as $\hat{\gama}_k = H_k\left[ \W^T\y + \b \right]$, and where $\b$ and $\b'$ are biases vectors. Note that while a convolutional multi-layer version of this family of autoencoders was proposed in \cite{makhzani2015winner}, these constructions are trained in stacked manner -- i.e., training the first layer independently, then training the second one to represent the features of the first layer while introducing pooling operations, and so forth. In this manner, each layer is trained to represent the (pooled) features from the previous layer, but the entire architecture cannot be directly employed for comparison in this problem.

Regarding the ML-CSC, we trained 6 different models by enforcing 6 different levels of sparsity in the convolutional filters (i.e., different values of the parameters $\zeta_i$ in Algorithm \ref{Alg:ML-CDL}), with a fixed target sparsity of $k=10$ non-zeros. The sparse coding of the inner-most $\hat{\gama}_3$ was done with the Iterative Hard Thresholding algorithm, in order to guarantee an exact number of non-zeros. The numbers pointing at the different models indicate the relative amount of parameters in the model, where 1 corresponds to $28^2\times 1K$ parameters required in a standard autoencoder (this is also the number of parameters in the sparse-autoencoders and k-sparse autoencoders, without counting the biases). As one can see, the larger the number of parameters, the lower the representation error the model is able to provide. In particular, the ML-CSC yields slightly better representation error than that of k-sparse autoencoders, for a wide range of non-zero values (without the need to train different models for each one) and \emph{with 1 and 2 orders of magnitude less parameters}.

Since the training of the ML-CSC model can also be understood as a dictionary learning algorithm, we compare here with the state-of-the-art method of \cite{Sulam2016}. For this case, we trained 1K trainlet atoms with the OSDL algorithm. Note that this comparison is interesting, as OSDL also provides sparse atoms with reduced number of parameters. For the sake of comparison, we employed an atom-sparsity that results in 13$\%$ of parameters relative to the total model size (just as one of the trained ML-CSC models), and the sparse coding was done also with the IHT algorithm. Notably, the performance of this relatively sophisticated dictionary learning method, which leverages the representation power of a cropped wavelets base dictionary, is only slightly superior to the proposed ML-CSC.

\subsection{Unsupervised Classification}

Unsupervised trained models are usually employed as feature extractors, and a popular way to assess the quality of such features is to train a linear classifier on them for a certain classification task. While the intention of this paper is not to provide a state-of-the-art unsupervised learning algorithm, we simply intent to demonstrate that the learned model generalizes to unseen examples, providing meaningful representations. To this end, we train a model with 3 layers, each containing: 16 ($5\times5$) atoms, 64 ($5\times5\times16$) atoms and 1024 atoms of dimension $5\times5\times64$ (stride of 2) on 60K training samples from MNIST. Just as for the previous model, the global sparse coding is performed with FISTA and a target (average) sparsity of 25 non-zeros. Once trained, we compute the representations $\hat{\gamma}_i$ with an elastic net formulation and non-negativity constraints, before fitting a simple linear classifier on the obtained features. 
Employing an elastic-net formulation (by including an $\ell_2$ regularization parameter, in addition to the $\ell_1$ norm) results in slightly denser representations, with improved classification performance. Similarly, the non-negativity constraint significantly facilitates the classification by linear classifiers. We compare our results with similar methods under the same experimental setup, and we depict the results in Table \ref{Table:UnsupervisedResults}, reporting the classification error on the 10K testing samples. 

Recall that within the ML-CSC model, all features $\gama_i$ have a very clear meaning: they provide a sparse representation at a different layer and scale. We can leverage this multi-layer decomposition in a very natural way within this unsupervised classification framework. We detail the classification performance achieved by our model in two different scenarios: on the first one we employ the
1K-dimensional features corresponding to the second layer of the ML-CSC model, obtaining better performance than the equivalent k-sparse autoencoder. In the second case, we add to the previous features the 1K-dimensional features from the third layer, resulting in a classification error of $1.15\%$, comparable to the Stacked Winner Take All (WTA) autoencoder (with the same number of neurons). 

\begin{table}[]
		\centering \small
		\begin{tabular}{|l|c|}
		\hline 
		\multicolumn{1}{|c|}{\textbf{Method}} & \textbf{Test Error} \\ \hline
		Stacked Denoising Autoencoder (3 layers) \cite{vincent2010stacked} & 1.28\%                        \\ \hline
		k-Sparse Autoencoder (1K units) \cite{makhzani2013k}		  & 1.35\%                        \\ \hline
		Shallow WTA Autoencoder (2K units) \cite{makhzani2015winner}  & 1.20\%                        \\ \hline
		Stacked WTA Autoencoder (2K units)\cite{makhzani2015winner}   & 1.11\%                        \\ \hline
		ML-CSC (1K units) - 2nd Layer Rep.                   & 1.30\%                        \\ \hline
		ML-CSC (2K units) - 2nd\&3rd Layer Rep.              & 1.15\%                        \\ \hline 
		\end{tabular}
		\caption{Unsupervised classification results on MNIST.}
		\label{Table:UnsupervisedResults}
	\end{table}

Lastly, it is worth mentioning that a stacked version of convolutional WTA autoencoder \cite{makhzani2015winner} achieve a classification error of 0.48, providing significantly better results. However, note that this model is trained with a 2-stage process (training the layers separately) involving significant pooling operations between the features at different layers. More importantly, the features computed by this model are 51,200-dimensional (more than an order of magnitude larger than in the other models) and thus cannot be directly compared to the results reporter by our method. In principle, similar stacked-constructions that employ pooling could be built for our model as well, and this remains as part of ongoing work.
 
\section{Conclusion}
\label{sec:conclusions}

We have carefully revisited the ML-CSC model and explored the problem of projecting a signal onto it. In doing so, we have provided new theoretical bounds for the solution of this problem as well as stability results for practical algorithms, both greedy and convex. The search for signals within the model led us to propose a simple, yet effective, learning formulation adapting the dictionaries across the different layers to represent natural images. 
We demonstrated the proposed approach on a number of practical applications, showing that the ML-CSC can indeed provide significant expressiveness with a very small number of model parameters. 

Several question remain open: how should the model be modified to incorporate pooling operations between the layers? what consequences, both theoretical and practical, would this have? How should one recast the learning problem in order to address supervised and semi-supervised learning scenarios? Lastly, we envisage that the analysis provided in this work will empower the development of better practical and theoretical tools not only for structured dictionary learning approaches, but to the field of deep learning and machine learning in general.

\section{Acknowledgments}
The research leading to these results has received funding from the European Research Council under European Unions Seventh Framework Programme, ERC Grant agreement no. 320649. J. Sulam kindly thanks J. Turek for fruitful discussions.



\bibliographystyle{ieeetr}
\bibliography{MyBib}

\begin{thebibliography}{10}

\bibitem{Bruckstein2009}
A.~M. Bruckstein, D.~L. Donoho, and M.~Elad, ``{From Sparse Solutions of
  Systems of Equations to Sparse Modeling of Signals and Images},'' {\em SIAM
  Review.}, vol.~51, pp.~34--81, Feb. 2009.

\bibitem{Rubinstein2010_dict}
R.~Rubinstein, A.~M. Bruckstein, and M.~Elad, ``Dictionaries for sparse
  representation modeling,'' {\em IEEE Proceedings - Special Issue on
  Applications of Sparse Representation \& Compressive Sensing}, vol.~98,
  no.~6, pp.~1045--1057, 2010.

\bibitem{Sulam2014}
J.~Sulam, B.~Ophir, and M.~Elad, ``{Image Denoising Through Multi-Scale Learnt
  Dictionaries},'' in {\em IEEE International Conference on Image Processing},
  pp.~808 -- 812, 2014.

\bibitem{Romano2014}
Y.~Romano, M.~Protter, and M.~Elad, ``Single image interpolation via adaptive
  nonlocal sparsity-based modeling,'' {\em IEEE Trans. on Image Process.},
  vol.~23, no.~7, pp.~3085--3098, 2014.

\bibitem{Mairal2009}
J.~Mairal, F.~Bach, and G.~Sapiro, ``{Non-local Sparse Models for Image
  Restoration},'' {\em IEEE International Conference on Computer Vision.},
  vol.~2, pp.~2272--2279, 2009.

\bibitem{Jiang2013}
Z.~Jiang, Z.~Lin, and L.~S. Davis, ``Label consistent k-svd: Learning a
  discriminative dictionary for recognition,'' {\em Pattern Analysis and
  Machine Intelligence, IEEE Transactions on}, vol.~35, no.~11, pp.~2651--2664,
  2013.

\bibitem{patel2014dictionaries}
V.~M. Patel, Y.-C. Chen, R.~Chellappa, and P.~J. Phillips, ``Dictionaries for
  image and video-based face recognition,'' {\em JOSA A}, vol.~31, no.~5,
  pp.~1090--1103, 2014.

\bibitem{shrivastava2014multiple}
A.~Shrivastava, V.~M. Patel, and R.~Chellappa, ``Multiple kernel learning for
  sparse representation-based classification,'' {\em IEEE Transactions on Image
  Processing}, vol.~23, no.~7, pp.~3013--3024, 2014.

\bibitem{lecun1990handwritten}
Y.~LeCun, B.~E. Boser, J.~S. Denker, D.~Henderson, R.~E. Howard, W.~E. Hubbard,
  and L.~D. Jackel, ``Handwritten digit recognition with a back-propagation
  network,'' in {\em Advances in neural information processing systems},
  pp.~396--404, 1990.

\bibitem{rumelhart1988learning}
D.~E. Rumelhart, G.~E. Hinton, R.~J. Williams, {\em et~al.}, ``Learning
  representations by back-propagating errors,'' {\em Cognitive modeling},
  vol.~5, no.~3, p.~1, 1988.

\bibitem{lecun2015deep}
Y.~LeCun, Y.~Bengio, and G.~Hinton, ``Deep learning,'' {\em Nature}, vol.~521,
  no.~7553, pp.~436--444, 2015.

\bibitem{bruna2013invariant}
J.~Bruna and S.~Mallat, ``Invariant scattering convolution networks,'' {\em
  IEEE transactions on pattern analysis and machine intelligence}, vol.~35,
  no.~8, pp.~1872--1886, 2013.

\bibitem{patel2015probabilistic}
A.~B. Patel, T.~Nguyen, and R.~G. Baraniuk, ``A probabilistic theory of deep
  learning,'' {\em arXiv preprint arXiv:1504.00641}, 2015.

\bibitem{cohen16Shashua}
N.~Cohen, O.~Sharir, and A.~Shashua, ``On the expressive power of deep
  learning: A tensor analysis,'' in {\em 29th Annual Conference on Learning
  Theory} (V.~Feldman, A.~Rakhlin, and O.~Shamir, eds.), vol.~49 of {\em
  Proceedings of Machine Learning Research}, (Columbia University, New York,
  New York, USA), pp.~698--728, PMLR, 23--26 Jun 2016.

\bibitem{gregor2010learning}
K.~Gregor and Y.~LeCun, ``Learning fast approximations of sparse coding,'' in
  {\em Proceedings of the 27th International Conference on Machine Learning
  (ICML-10)}, pp.~399--406, 2010.

\bibitem{xin2016maximal}
B.~Xin, Y.~Wang, W.~Gao, D.~Wipf, and B.~Wang, ``Maximal sparsity with deep
  networks?,'' in {\em Advances in Neural Information Processing Systems},
  pp.~4340--4348, 2016.

\bibitem{Papyan2016convolutional}
V.~Papyan, Y.~Romano, and M.~Elad, ``Convolutional neural networks analyzed via
  convolutional sparse coding,'' {\em The Journal of Machine Learning
  Research}, vol.~18, no.~1, pp.~2887--2938, 2017.

\bibitem{Lemagoarou15}
L.~Le~Magoarou and R.~Gribonval, ``{Chasing butterflies: In search of efficient
  dictionaries},'' in {\em IEEE Int. Conf. Acoust. Speech, Signal Process},
  Apr. 2015.

\bibitem{Chabiron2013}
O.~Chabiron, F.~Malgouyres, J.~Tourneret, and N.~Dobigeon, ``{Toward Fast
  Transform Learning},'' {\em International Journal of Computer Vision},
  pp.~1--28, 2015.

\bibitem{Sulam2016}
J.~Sulam, B.~Ophir, M.~Zibulevsky, and M.~Elad, ``Trainlets: Dictionary
  learning in high dimensions,'' {\em IEEE Transactions on Signal Processing},
  vol.~64, no.~12, pp.~3180--3193, 2016.

\bibitem{ng2011sparse}
A.~Ng, ``Sparse autoencoder,'' {\em CS294A Lecture notes}, vol.~72, no.~2011,
  pp.~1--19, 2011.

\bibitem{makhzani2013k}
A.~Makhzani and B.~Frey, ``K-sparse autoencoders,'' {\em arXiv preprint
  arXiv:1312.5663}, 2013.

\bibitem{makhzani2015winner}
A.~Makhzani and B.~J. Frey, ``Winner-take-all autoencoders,'' in {\em Advances
  in Neural Information Processing Systems}, pp.~2791--2799, 2015.

\bibitem{WorkingLocallyThinkingGlobally}
V.~Papyan, J.~Sulam, and M.~Elad, ``Working locally thinking globally:
  Theoretical guarantees for convolutional sparse coding,'' {\em IEEE
  Transactions on Signal Processing}, vol.~65, no.~21, pp.~5687--5701, 2017.

\bibitem{sermanet2013pedestrian}
P.~Sermanet, K.~Kavukcuoglu, S.~Chintala, and Y.~LeCun, ``Pedestrian detection
  with unsupervised multi-stage feature learning,'' in {\em Computer Vision and
  Pattern Recognition (CVPR), 2013 IEEE Conference on}, pp.~3626--3633, IEEE,
  2013.

\bibitem{li2013convolutional}
K.~Li, L.~Gan, and C.~Ling, ``Convolutional compressed sensing using
  deterministic sequences,'' {\em IEEE Transactions on Signal Processing},
  vol.~61, no.~3, pp.~740--752, 2013.

\bibitem{zhang2016convolutional}
H.~Zhang and V.~M. Patel, ``Convolutional sparse coding-based image
  decomposition.,'' in {\em BMVC}, 2016.

\bibitem{zhang2018convolutional}
H.~Zhang and V.~M. Patel, ``Convolutional sparse and low-rank coding-based
  image decomposition,'' {\em IEEE Transactions on Image Processing}, vol.~27,
  no.~5, pp.~2121--2133, 2018.

\bibitem{Papyan_2017_ICCV}
V.~Papyan, Y.~Romano, J.~Sulam, and M.~Elad, ``Convolutional dictionary
  learning via local processing,'' in {\em The IEEE International Conference on
  Computer Vision (ICCV)}, Oct 2017.

\bibitem{Heide2015}
F.~Heide, W.~Heidrich, and G.~Wetzstein, ``Fast and flexible convolutional
  sparse coding,'' in {\em Computer Vision and Pattern Recognition (CVPR), 2015
  IEEE Conference on}, pp.~5135--5143, IEEE, 2015.

\bibitem{choudhury2017consensus}
B.~Choudhury, R.~Swanson, F.~Heide, G.~Wetzstein, and W.~Heidrich, ``Consensus
  convolutional sparse coding,'' in {\em Computer Vision (ICCV), 2017 IEEE
  International Conference on}, pp.~4290--4298, IEEE, 2017.

\bibitem{henaff2011unsupervised}
M.~Henaff, K.~Jarrett, K.~Kavukcuoglu, and Y.~LeCun, ``Unsupervised learning of
  sparse features for scalable audio classification.,'' in {\em ISMIR},
  vol.~11, p.~2011, Citeseer, 2011.

\bibitem{szlam2011structured}
A.~D. Szlam, K.~Gregor, and Y.~L. Cun, ``Structured sparse coding via lateral
  inhibition,'' in {\em Advances in Neural Information Processing Systems},
  pp.~1116--1124, 2011.

\bibitem{Wohlberg2016}
B.~Wohlberg, ``Efficient algorithms for convolutional sparse representations,''
  {\em IEEE Transactions on Image Processing}, vol.~25, pp.~301--315, Jan.
  2016.

\bibitem{liu2017online}
J.~Liu, C.~Garcia-Cardona, B.~Wohlberg, and W.~Yin, ``Online convolutional
  dictionary learning,'' {\em arXiv preprint arXiv:1709.00106}, 2017.

\bibitem{zeiler2010deconvolutional}
M.~D. Zeiler, D.~Krishnan, G.~W. Taylor, and R.~Fergus, ``Deconvolutional
  networks,'' in {\em Computer Vision and Pattern Recognition (CVPR), 2010 IEEE
  Conference on}, pp.~2528--2535, IEEE, 2010.

\bibitem{szlam2010convolutional}
A.~Szlam, K.~Kavukcuoglu, and Y.~LeCun, ``Convolutional matching pursuit and
  dictionary training,'' {\em arXiv preprint arXiv:1010.0422}, 2010.

\bibitem{kavukcuoglu2010learning}
K.~Kavukcuoglu, P.~Sermanet, Y.-L. Boureau, K.~Gregor, M.~Mathieu, and Y.~L.
  Cun, ``Learning convolutional feature hierarchies for visual recognition,''
  in {\em Advances in neural information processing systems}, pp.~1090--1098,
  2010.

\bibitem{he2014unsupervised}
Y.~He, K.~Kavukcuoglu, Y.~Wang, A.~Szlam, and Y.~Qi, ``Unsupervised feature
  learning by deep sparse coding,'' in {\em Proceedings of the 2014 SIAM
  International Conference on Data Mining}, pp.~902--910, SIAM, 2014.

\bibitem{Candes2005}
E.~J. Candes and T.~Tao, ``Decoding by linear programming,'' {\em Information
  Theory, IEEE Transactions on}, vol.~51, no.~12, pp.~4203--4215, 2005.

\bibitem{liu2015sparse}
B.~Liu, M.~Wang, H.~Foroosh, M.~Tappen, and M.~Pensky, ``Sparse convolutional
  neural networks,'' in {\em Proceedings of the IEEE Conference on Computer
  Vision and Pattern Recognition}, pp.~806--814, 2015.

\bibitem{Blumensath2008}
T.~Blumensath and M.~E. Davies, ``{Iterative Thresholding for Sparse
  Approximations},'' {\em Journal of Fourier Analysis and Applications},
  vol.~14, pp.~629--654, Sept. 2008.

\bibitem{le2016flexible}
L.~Le~Magoarou and R.~Gribonval, ``Flexible multilayer sparse approximations of
  matrices and applications,'' {\em IEEE Journal of Selected Topics in Signal
  Processing}, vol.~10, no.~4, pp.~688--700, 2016.

\bibitem{Rubinstein2010}
R.~Rubinstein, M.~Zibulevsky, and M.~Elad, ``{Double Sparsity : Learning Sparse
  Dictionaries for Sparse Signal Approximation},'' {\em IEEE Trans. Signal
  Process.}, vol.~58, no.~3, pp.~1553--1564, 2010.

\bibitem{MNIST}
Y.~LeCun, L.~Bottou, Y.~Bengio, and P.~Haffner, ``Gradient-based learning
  applied to document recognition,'' {\em Proceedings of the IEEE}, vol.~86,
  no.~11, pp.~2278--2324, 1998.

\bibitem{beck2009fast}
A.~Beck and M.~Teboulle, ``A fast iterative shrinkage-thresholding algorithm
  for linear inverse problems,'' {\em SIAM journal on imaging sciences},
  vol.~2, no.~1, pp.~183--202, 2009.

\bibitem{dai2009subspace}
W.~Dai and O.~Milenkovic, ``Subspace pursuit for compressive sensing signal
  reconstruction,'' {\em IEEE Transactions on Information Theory}, vol.~55,
  no.~5, pp.~2230--2249, 2009.

\bibitem{Elad_Book}
M.~Elad, {\em Sparse and Redundant Representations: From Theory to Applications
  in Signal and Image Processing}.
\newblock Springer Publishing Company, Incorporated, 1st~ed., 2010.

\bibitem{vincent2010stacked}
P.~Vincent, H.~Larochelle, I.~Lajoie, Y.~Bengio, and P.-A. Manzagol, ``Stacked
  denoising autoencoders: Learning useful representations in a deep network
  with a local denoising criterion,'' {\em Journal of Machine Learning
  Research}, vol.~11, no.~Dec, pp.~3371--3408, 2010.

\end{thebibliography}


\appendix 

\newtheorem{innercustomgeneric}{\customgenericname}
\providecommand{\customgenericname}{}
\newcommand{\newcustomtheorem}[2]{%
	\newenvironment{#1}[1]
	{%
		\renewcommand\customgenericname{#2}%
		\renewcommand\theinnercustomgeneric{##1}%
		\innercustomgeneric
	}
	{\endinnercustomgeneric}
}

\newcustomtheorem{customlemma}{Lemma}
\newcustomtheorem{customdef}{Definition}

\subsection{Properties of the ML-CSC model}
\label{app:MLCSCisCSC}

\begin{customlemma}{1} \label{lemma:MLCSCisCSC}
	Given the ML-CSC model described by the set of convolutional dictionaries $\{\D_i\}_{i=1}^L$, with filters of spatial dimensions $n_i$ and channels $m_i$, any dictionary $\D^{(i)} = \D_1 \D_2 \dots \D_i$ is a convolutional dictionary with $m_i$ local atoms of dimension $n_i^{\text{eff}} = \sum_{j=1}^{i} n_j - (i-1)$. In other words, the ML-CSC model is a structured global convolutional model.
\end{customlemma}
\begin{proof} 
	A convolutional dictionary is formally defined as the concatenation of banded circulant matrices. Consider $\D_1 = \left[ \C^{(1)}_1, \C^{(1)}_2, \dots, \C^{(1)}_{m_1}\right]$, where each circulant $\C^{(1)}_i \in \mathbb{R}^{N\times N}$. Likewise, one can express $\D_2 = \left[ \C^{(2)}_1, \C^{(2)}_2, \dots, \C^{(2)}_{m_2}\right]$, where $\C^{(2)}_i \in \mathbb{R}^{Nm_1\times N}$. Then,
	\begin{equation}
	\D^{(2)} = \D_1\D_2 = \left[ \D_1 \C^{(2)}_1, \D_1\C^{(2)}_2, \dots, \D_1\C^{(2)}_{m_2}\right].
	\end{equation}
	Each term $\D_1\C^{(2)}_i$ is the product of a concatenation of banded circulant matrices and a banded circulant matrix. Because the atoms in each $\C^{(2)}_i$ have a stride of $m_1$ (the number of filters in $\D_1$) each of these products is in itself a banded circulant matrix. This is illustrated in Figure \ref{fig:proof_csc}, where it becomes clear that the first atom in $\C^{(2)}_1$ (of length $n_2m_1$) linearly combines atoms from the first $n_2$ blocks of $m_1$ filters in $\D_1$ (in this case $n_2 = 2$). These block are simply the unique set of filters shifted at every position. The second column in $\C^{(2)}_1$ will do the same for the next set $n_2$ blocks, starting from the second one, etc. 
	
	From the above discussion, $\D_1\C^{(2)}_1$ results in a banded circulant matrix of dimension $N\times N$. In particular, the band of this matrix is given by the dimension of the filters in the first dictionary ($n_1$) plus the number of blocks combined by $\C^{(2)}_1$ minus one. In other words, the effective dimension of the filters in $\D_1\C^{(2)}_1$ is given by $n^\text{eff}_2 = n_2+n_1-1$.
	
	The effective dictionary $\D^{(2)} = \D_1\D_2$ is simply a concatenation of $m_2$ such banded circulant matrices. In other words, $\D^{(2)}$ is a convolutional dictionary with filters of dimension $n_2^\text{eff}$. The same analysis can be done for the effective dictionary at every layer, $\D^{(i)}$, resulting in an effective dimension of $n^{\text{eff}}_i = n_i + n^{\text{eff}}_{i-1}-1$, and so $n_L^{\text{eff}} = \sum_{i=1}^{L} n_i - (L-1)$. 
	
	
	\begin{figure}
		\begin{center}
			\includegraphics[trim = 100 135 100 80, width = .45\textwidth]{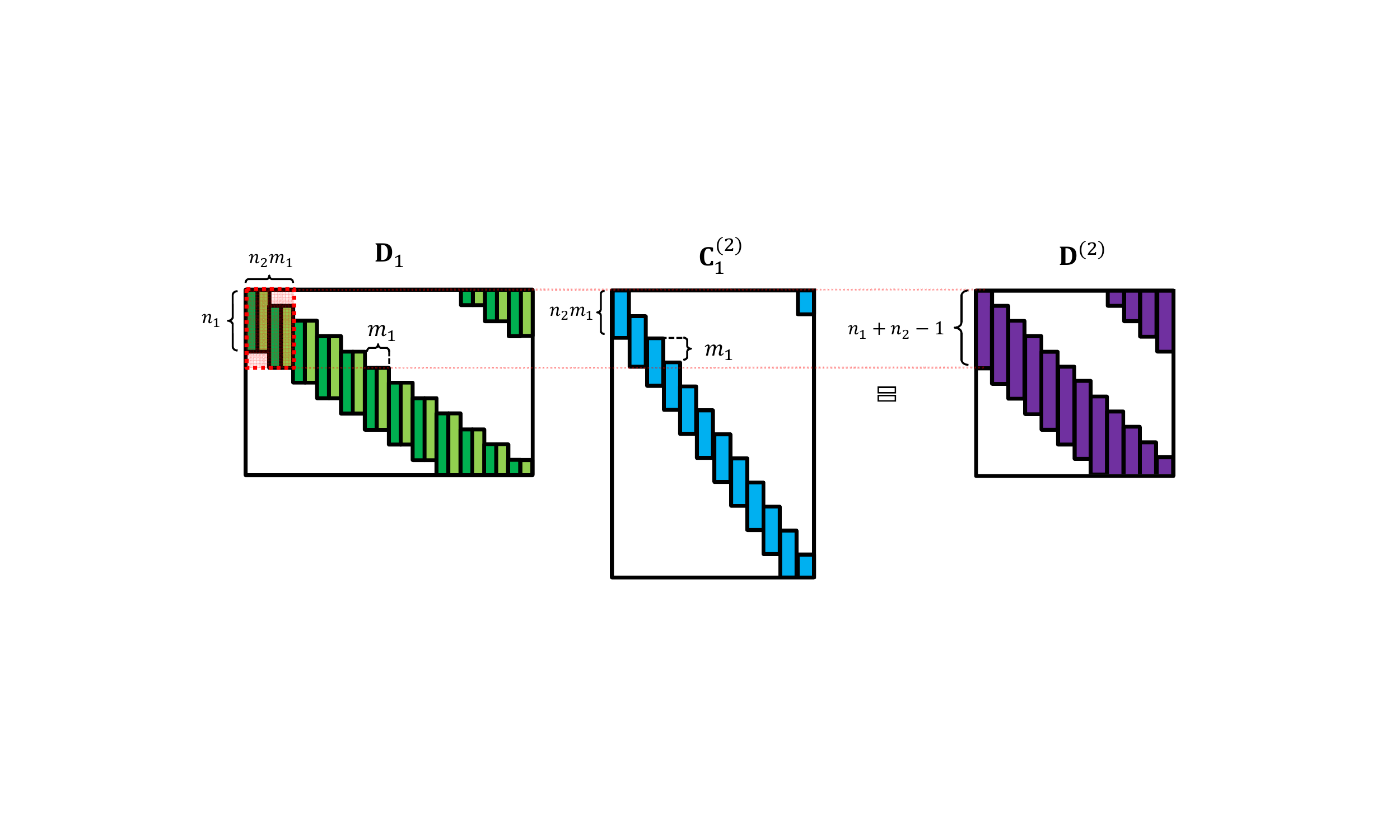}
		\end{center}
		\caption{Illustration of a convolutional dictionary $\D_1$ multiplied by one of the circulat matrices from $\D_2$, in this case $\C^{(2)}_1$.}
		\label{fig:proof_csc}
	\end{figure}
	
	
	Finally, note that $\D^{(i)}$ has $N m_i$ columns, and thus there will be $m_i$ local filters in the effective CSC model.
\end{proof}

\subsection{Stability result for the $\PM$ problem}
\label{app:StabilityforPM}
\begin{customthm}{4}{ Stability of the solution to the $\PM$ problem:} \label{Thm:GlobalStability} \\
	Suppose $\x(\gama_i) \in \M_\lamda$ is observed through $\y = \x+ \v$, where $\v$ is a bounded noise vector, $\|\v\|_2 \leq \mathcal{E}_0$, and 
	$\|\gama_i\|^s_{0,\infty} = \lambda_i < \frac{1}{2}\left(1+\frac{1}{\mu(\D^{(i)})}\right)$, for $1\leq i \leq L$. Consider the set $\{\hat{\gama}_i\}_{i=1}^{L}$ to be the solution of the $\PM$ problem. Then,
	\begin{equation} \label{Eq:DCPEStability}
	\| \gama_i- \hat{\gama}_i \|_2^2 \leq \frac{4\mathcal{E}_{0}^2}{1-(2\|\gama_{i}\|^s_{0,\infty}-1)\mu(\D^{(i)})} .
	\end{equation}
\end{customthm}

\begin{proof}
	Denote the solution to the $\PM$ problem by $\hat{\x}$; i.e., $\hat{\x} = \D^{(i)}\hat{\gama}_i$.  Given that the original signal $\x$ satisfies $\|\y-\x\|_2\leq \mathcal{E}_0$,  the solution to the $\PM$ problem, $\hat{\x}$ must satisfy 
	\begin{equation}
	\|\y-\hat{\x}\|_2\leq \|\y- \x \|_2\leq \mathcal{E}_0,
	\end{equation}
	as this is the signal which provides the shortest $\ell_2$ (data-fidelity) distance from $\y$. Note that because $\hat{\x}(\gama_i)\in \M_{\lamda}$, we can have that $\hat{\x} = \D^{(i)}\hat{\gama}_i$, $\forall\ 1\leq i \leq L$. Recalling Lemma \ref{lemma:MLCSCisCSC}, the product $\D_1\D_2\dots\D_i$ is a convolutional dictionary. In addition, we have required that $\|\hat{\gama}_i\|^s_{0,\infty} \leq \lambda_i < \frac{1}{2}\left(1+\frac{1}{\mu(\D^{(i)})}\right) $. Therefore, from the same arguments presented in \cite{WorkingLocallyThinkingGlobally}, it follows that
	\begin{equation}
	\|\gama_i-\hat{\gama}_i\|_2^2\leq \frac{4\mathcal{E}_0^2}{1-(2\|\gama_i\|^s_{0,\infty}-1)\mu(\D^{(i)})}.
	\end{equation}
	
\end{proof}

\label{app:AnotherStabilityforPM}
\begin{thm}{(Another stability of the solution to the $\PM$ problem):}\\
	Suppose $\x(\gama_i) \in \M_\lamda$ is observed through $\y = \x+ \v$, where $\v$ is a bounded noise vector, $\|\v\|_2 \leq \mathcal{E}_0$, and \footnote{The assumption that $\|\gama_i\|^s_{0,\infty} = \lambda_i$ can be relaxed to $\|\gama_i\|^s_{0,\infty} \leq \lambda_i$, with slight modifications of the result.} $\|\gama_i\|^s_{0,\infty} = \lambda_i < \frac{1}{2}\left(1+\frac{1}{\mu(\D_i)}\right)$, for $1\leq i \leq L$. Consider the set $\{\hat{\gama}_i\}_{i=1}^{L}$ to be the solution of the $\PM$ problem. If $\| {\gama}_L \|^s_{0,\infty} < \frac{1}{2} \left( 1 + \frac{1}{\mu(\D^{(L)})}\right)$ then 
	\begin{multline} \label{Eq:DCPEStability}
	\| \gama_i- \hat{\gama}_i \|_2^2 \leq \frac{4\mathcal{E}_{0}^2}{1-(2\|\gama_{L}\|^s_{0,\infty}-1)\mu(\D^{(L)})} \\ \prod_{j=i+1}^{L} \left[1 + (2\|\gama_{j}\|^s_{0,\infty} -1)\mu(\D_{j}) \right].
	\end{multline}
\end{thm}

\begin{proof}
	Given that the original signal $\x$ satisfies $\|\y-\x\|_2\leq \mathcal{E}_0$,  the solution to the $\PM$ problem, $\hat{\x}$ must satisfy 
	\begin{equation}
	\|\y-\hat{\x}\|_2\leq \|\y- \x \|_2\leq \mathcal{E}_0,
	\end{equation}
	as this is the signal which provides a lowest $\ell_2$ (data-fidelity) term. In addition, $\|\hat{\gama}_L\|^s_{0,\infty} = \lambda_L < \frac{1}{2}(1+\frac{1}{\mu(\D^{(L)})}) $. Therefore, from the same arguments presented in \cite{WorkingLocallyThinkingGlobally}, it follows that
	\begin{equation}
	\|\gama_L-\hat{\gama}_L\|_2^2\leq \frac{4\mathcal{E}_0^2}{1-(2\|\gama_L\|^s_{0,\infty}-1)\mu(\D^{(L)})} = \mathcal{E}^2_L.
	\end{equation}
	Because the solution $\hat{\x}(\{\hat{\gama}_i\})\in\M_\lamda$, then $\hat{\gama}_{L-1} = \D_L \hat{\gama}_L$. Therefore
	\begin{equation}
	\|\gama_{L-1}-\hat{\gama}_{L-1}\|_2^2 = \|\D_L (\gama_{L}-\hat{\gama}_{L})\|_2^2 \leq (1+\delta_{2k})\|\gama_{L}-\hat{\gama}_{L}\|_2^2,
	\end{equation}
	where $\delta_{2k}$ is the S-RIP of $\D_L$ with constant $2k = 2\|\gama_L\|^s_{0,\infty}$. This follows from the triangle inequality of the $\Loi$ norm and the fact that, because $\hat{\gama}_L$ is a solution to the $\PM$ problem, $\|\hat{\gama}_L\|^s_{0,\infty} \leq \lambda_L = \|{\gama}_L\|^s_{0,\infty}$. The S-RIP can in turn be bounded with the mutual coherence \cite{WorkingLocallyThinkingGlobally} as $\delta_k \leq (k-1) \mu(\D_L)$, from which one obtains
	\begin{equation}
	\|\gama_{L-1}-\hat{\gama}_{L-1}\|_2^2 \leq \mathcal{E}^2_L\ (1+(2\|\gama_L\|^s_{0,\infty}-1)\mu(\D_L)).
	\end{equation}
	From similar arguments, extending this to an arbitrary $i^{th}$ layer, 
	\begin{equation}
	\| \gama_i- \hat{\gama}_i \|_2^2 \leq \mathcal{E}^2_L \prod_{j=i+1}^{L} (1 + (2\|\gama_j\|^s_{0,\infty} -1)\mu(\D_{j})).
	\end{equation}
\end{proof}

For the sake of simplicity, one can relax the above bounds further obtaining that, subject to the assumptions in Theorem \ref{Thm:GlobalStability}, 
\begin{equation} \label{Eq:DCPEStability_simple}
\| \gama_i- \hat{\gama}_i \|_2^2 \leq \mathcal{E}^2_L\ 2^{(L-i)}.
\end{equation}
This follows simply by employing the fact that $\| \gama_i \|^s_{0,\infty} < \frac{1}{2} \left( 1 + \frac{1}{\mu(\D_i)} \right)$.

\subsection{Local stability of the S-RIP}
\label{app:LocalStabilitySRIP}

\begin{customlemma}{2}{Local one-sided near isometry:} \label{lemma:LocalSRIP} \\
	If $\D$ is a convolutional dictionary satisfying the Stripe-RIP condition with constant $\delta_k$, then
	\begin{equation}
	\|\D\gama\|^{2,p}_{2,\infty} \leq (1+\delta_k)\ \|\gama\|^{2,s}_{2,\infty}
	\end{equation}
\end{customlemma}

\begin{proof} 
	Consider the patch-extraction operator $\mathbf{P}_i$ from the signal $\x = \D\gama$, and $\mathbf{S}_i$ the operator that extracts the corresponding stripe from $\gama$ such that $\mathbf{P}_i \x = \O \mathbf{S}_i \gama$, where $\O$ is a local stripe dictionary \cite{WorkingLocallyThinkingGlobally}. Denote the $i^{th}$ stripe by $\mathbf{s}_i = \mathbf{S}_i \gama$. Furthermore, denote by $\bar{\mathbf{S}}_i$ the operator that \emph{extracts the support} of $\mathbf{s}_i$ from $\gama$. Clearly, $ \x = \D \bar{\mathbf{S}}_i^T \bar{\mathbf{S}}_i \gama$. Note that $\| \mathbf{P}_i \|_2 = \|\mathbf{S}_i\|_2 = 1$.
	Then,
	\begin{align}
	\|\D\gama\|^p_{2,\infty}  & =  \underset{i}{\max} \|\mathbf{P}_i\D \bar{\mathbf{S}}_i^T \bar{\mathbf{S}}_i \gama\|_2 \\
	& \leq  \underset{i}{\max} \|\mathbf{P}_i\|_2 \ \| \D\bar{\mathbf{S}}^T_i \bar{\mathbf{S}}_i \gama\|_2 \\
	& \leq  \underset{i}{\max} \| \D\bar{\mathbf{S}}_i^T \|_2 \| \bar{\mathbf{S}}_i \gama\|_2 \\
	& \leq  \underset{i}{\max} \| \D\bar{\mathbf{S}}_i^T \|_2 \ \underset{j}{\max} \| \bar{\mathbf{S}}_j \gama\|_2.
	\end{align}
	
	Note that 
	\begin{equation}
	\underset{j}{\max} \| \bar{\mathbf{S}}_j \gama\|_2 = \underset{j}{\max} \| \mathbf{S}_j \gama\|_2 = \| \gama\|^s_{2,\infty},
	\end{equation}
	as the non-zero entries in $\bar{\mathbf{S}}_j \gama$ and $\mathbf{S}_j \gama$ are the same. On the other hand, denoting by $\lambda_{max}(\cdot)$ the maximal eigenvalue of the matrix in its argument, $\| \D\bar{\mathbf{S}}_i^T \|_2 = \sqrt{\lambda_{max}\left( \bar{\mathbf{S}}_i \D^T\D \bar{\mathbf{S}}_i^T \right)}$, and if $\mathcal{T} = Supp(\gama)$,
	\begin{equation} \label{eq:EigenInequality}
	\lambda_{max}\left( \bar{\mathbf{S}}_i \D^T\D \bar{\mathbf{S}}_i^T \right) \leq \lambda_{max}\left( \D_\mathcal{T}^T\D_\mathcal{T} \right),
	\end{equation}
	because\footnote{The inequality in \eqref{eq:EigenInequality} can be shown by considering the equivalent expression $\lambda_{max}\left( \mathbf{S}_i\D^T_\mathcal{T}\D_\mathcal{T}\mathbf{S}^T_i\right)$, where the matrix $\D^T_\mathcal{T}\D_\mathcal{T}$ is real and symmetric, and the matrix $\mathbf{S}_i$ is semi-orthogonal; i.e. $\mathbf{S}_i \mathbf{S}^T_i = \mathbf{I}$. Thus, from Poincar\'e Separation Theorem, $\lambda_{min}\left( \D^T_\mathcal{T}\D_\mathcal{T} \right) \leq \lambda\left( \mathbf{S}_i\D^T_\mathcal{T}\D_\tau\mathbf{S}^T_i\right) \leq \lambda_{max}\left( \D^T_\mathcal{T}\D_\mathcal{T} \right)$.} $ \bar{\mathbf{S}}_i \D^T\D \bar{\mathbf{S}}_i^T $ is a principal sub-matrix of $\D_\mathcal{T}^T\D_\mathcal{T}$. Thus, also $\|\D\bar{\mathbf{S}}^T_i\|_2 \leq \| \D_\mathcal{T}\|_2$. 
	
	The Stripe-RIP condition, as in Equation \eqref{Eq:SRIP}, provides a bound on the square of the singular values of $\D_\mathcal{T}$. Indeed, $\| \D_\mathcal{T} \|^2_2 \leq (1+\delta_k)$, for every $\mathcal{T} : \|\mathcal{T}\|^s_{0,\infty} = k$. Including these in the above one obtains the desired claim:
	\begin{equation}
	\|\D\gama\|^p_{2,\infty}  \leq \underset{i}{\max} \| \D\bar{\mathbf{S}}_i^T \|_2 \ \underset{j}{\max} \| \bar{\mathbf{S}}_j \gama\|_2 \leq \sqrt{1+\delta_k} \| \gama \|^s_{2,\infty}.
	\end{equation}
\end{proof}

\subsection{Recovery guarantees for pursuit algorithms}

\subsubsection{Convex relaxation case}
\label{app:GuaranteesConvexRelaxation}

\begin{customthm}{6}{Stable recovery of the Multi-Layer Pursuit Algorithm in the convex relaxation case:} \label{Thm:StabilityPursuitLasso} \\
	Suppose a signal $\x(\gama_i) \in \M_\lamda$ is contaminated with locally-bounded noise $\v$, resulting in $\y = \x + \v$, $\|\v\|^p_{2,\infty} \leq \epsilon_0$. Assume that all representations $\gama_i$ satisfy the N.V.S. property for the respective dictionaries $\D_i$, and that $\|\gama_i\|^s_{0,\infty} = \lambda_i < \frac{1}{2}\left(1+\frac{1}{\mu(\D_i)}\right)$, for $1\leq i \leq L$ and $\|\gama_L\|^s_{0,\infty} = \lambda_L \leq \frac{1}{3}\left(1+\frac{1}{\mu(\D^{(L)})}\right)$. Consider solving the Pursuit stage in Algorithm \ref{Alg:MulilayerPursuit} as
	\begin{equation}
	\hat{\gama}_L = \underset{\gama}{\arg\min} \| \y + \D^{(L)} \gama \||^2_2 + \zeta_L \|\gama\|_1,
	\end{equation}
	for $\zeta_L = 4\epsilon_0$, and set $\hat{\gama}_{i-1} = \D_i\hat{\gama}_i$, $i=L,\dots,1$. 
	\noindent
	Then, for every $1\leq i\leq L$ layer,
	\begin{enumerate}
		\item $Supp(\hat{\gama}_i) \subseteq Supp(\gama_i)$,
		\item $\|\hat{\gama}_i - \gama_i\|^p_{2,\infty} \leq \epsilon_L  \displaystyle\prod\limits_{j=i+1}^{L} \sqrt{\frac{3 c_j}{2}}$,
	\end{enumerate}
	where $\epsilon_L = \frac{15}{2}\ \epsilon_0 \sqrt{\|\gama_j\|^p_{0,\infty}}$ is the error at the last layer, and  $c_j$ is a coefficient that depends on the ratio between the local dimensions of the layers, $c_j = \Bigl\lceil \frac{2n_{j-1}-1}{n_j} \Bigr\rceil$.
\end{customthm}

\begin{proof}
	Denote $ \Delt_i = \hat{\gama}_i - \gama_i $.
	From \cite{WorkingLocallyThinkingGlobally} (Theorem 19), the solution $\hat{\gama}_L$ will satisfy:
	\begin{enumerate}[i)]
		\item $\S(\hat{\gama}_L) \subseteq \S(\gama_L)$; and
		\item $\|\Delt_L\|_\infty \leq  \frac{15}{2}\ \epsilon_0$.
	\end{enumerate}
	
	As shown in \cite{Papyan2016convolutional}, given the $\ell_\infty$ bound of the representation error, we can bound its $\ell_{2,\infty}$ norm as well, obtaining
	\begin{equation}
	\label{Eq:BoundOnEl}
	\| \Delt_L \|^p_{2,\infty} \leq \|\Delt_L\|_\infty \sqrt{\|\Delt_L\|^p_{0,\infty}} \leq \frac{15}{2}\ \epsilon_0 \sqrt{\|\gama_L\|^p_{0,\infty}},
	\end{equation}
	because, since $\S(\hat{\gama}_L) \subset \S(\gama_L)$, $\| \Delt_{L}\|^s_{0,\infty} \leq \|\gama_L\|^s_{0,\infty}$.	 Define $\epsilon_L = \frac{15}{2}\ \epsilon_0 \sqrt{\|\gama_L\|^p_{0,\infty}}$.
	
	Recall that the N.V.S. property states that the entries in $\gama$ will no cause the support of the atoms in $\D$ cancel each other; i.e., $\|\D\gama\|_0 = \|\D_{\mathcal{T}}\|^0_{\infty}$ (Definition \ref{def:N.V.S.Property}). In other words, this provides a bound on the cardinality of the vector resulting from the multiplication of $\D$ with any sparse vector with support $\mathcal{T}$. Concretely, if $\gama$ satisfies the N.V.S., then $\|\D\gama\|_0 \geq \|\D\hat{\gama}\|_0$.
	
	Consider now the estimate at the $L-1$ layer, obtained as $\hat{\gama}_{L-1} = \D_L \hat{\gama}_{L}$. Because $\gama_L$ satisfies the N.V.S. property, and $\S(\hat{\gama}_L) \subseteq \S(\gama_L)$, then $\|\hat{\gama}_{L-1}\|_0 \leq \|\gama_{L-1}\|_0$, and more so  $\S(\hat{\gama}_{L-1}) \subseteq \S(\gama_{L-1})$.

	On the other hand, recalling Lemma \ref{lemma:LocalSRIP} and denoting by $\delta_{\lambda_L}$ the Stripe-RIP constant of the $\D_L$ dictionary, and because $\| \Delt_{L}\|^s_{0,\infty} \leq \|\gama_L\|^s_{0,\infty} \leq \lambda_L$,
	\begin{equation}
	\| \Delt_{L-1} \|^{2,p}_{2,\infty} = \| \D_L \Delt_{L} \|^{2,p}_{2,\infty} \leq (1+\delta_{\lambda_L}) \| \Delt_{L} \|^{2,s}_{2,\infty}.
	\end{equation}
	
	Notice that by employing the above Lemma, we have bounded the \textbf{patch-wise} $\ell_{2,\infty}$ norm of $\Delt_{L-1}$ in terms of the \textbf{stripe-wise} $\ell_{2,\infty}$ of $\Delt_{L}$. 
	Recalling the derivation from \cite{Papyan2016convolutional} (Section 7.1), at each $i^{th}$ layer, a stripe includes up to $(2n_{i-1}-1)/n_i$ patches. Define $c_i = \Bigl\lceil \frac{2n_{i-1}-1}{n_i} \Bigr\rceil$. From this, one can bound the square of the $\ell_2$ norm of a stripe with the norm of the maximal patch within it - this is true for every stripe, and in particular for the stripe with the maximal norm. This implies that $\| \Delt_{L} \|^{2,s}_{2,\infty} \leq c_{L} \| \Delt_{L} \|^{2,p}_{2,\infty}$. Then, 
	\begin{equation}
	\| \Delt_{L-1} \|^{2,p}_{2,\infty} \leq (1+\delta_k) \| \Delt_{L} \|^{2,s}_{2,\infty} \leq (1+\delta_{\lambda_L}) c_{L} \| \Delt_{L} \|^{2,p}_{2,\infty}.
	\end{equation}
	Employing the result in Eq. \eqref{Eq:BoundOnEl},
	\begin{equation}
	\| \Delt_{L-1} \|^{2,p}_{2,\infty} \leq (1+\delta_k) c_L \| \Delt_{L} \|^{2,p}_{2,\infty} \leq  (1+\delta_k)\ c_L \ \epsilon_L^2.
	\end{equation}
	We can further bound the Stripe-RIP constant by $\delta_k \leq (k-1)\mu(\D)$ \cite{WorkingLocallyThinkingGlobally}, obtaining
	\begin{equation}
	\| \Delt_{L-1} \|^{2,p}_{2,\infty} \leq (1+ (\|\gama_L\|^s_{0,\infty}-1)\mu(\D_L) ) \ \epsilon^2_L \ c_L.
	\end{equation}
	Iterating this analysis for the remaining layers yields
	\begin{equation}
	\|\hat{\gama}_i - \gama_i\|^{2,p}_{2,\infty} \leq \epsilon_L^2  \displaystyle\prod\limits_{j=i+1}^{L} c_j\ (1 + (\|\gama_j\|^s_{0,\infty} -1)\mu(\D_{j})).
	\end{equation}
	
	This general result can be relaxed for the sake of simplicity. Indeed, considering that $\|\gama_i\|^s_{0,\infty} < \frac{1}{2}\left(1+\frac{1}{\mu(\D_i)}\right)$, for $1\leq i \leq L$, 
	\begin{equation}
	1 + (\|\gama_j\|^s_{0,\infty} -1)\mu(\D_{j}) < 3/2,
	\end{equation}
	and so 
	\begin{equation}
	\|\hat{\gama}_i - \gama_i\|^p_{2,\infty} \leq \epsilon_L  \displaystyle\prod\limits_{j=i+1}^{L} \sqrt{\frac{3 c_j}{2}}
	\end{equation}
\end{proof}

\subsubsection{Greedy case}
\label{app:StableGuaranteesGreedy}

\begin{customthm}{7}{Stable recovery of the Multi-Layer Pursuit Algorithm in the greedy case:} \label{Thm:StabilityPursuitOMP} \\
	Suppose a signal $\x(\gama_i) \in \M_\lamda$ is contaminated with energy-bounded noise $\v$, such that $\y = \x + \v$, $\|\y-\x\|_2 \leq \mathcal{E}_0$, and $\epsilon_0 = \|\v\|^\pp_{2,\infty}$. Assume that all representations $\gama_i$ satisfy the N.V.S. property for the respective dictionaries $\D_i$, with $\|\gama_i\|^s_{0,\infty} = \lambda_i < \frac{1}{2}\left(1+\frac{1}{\mu(\D_i)}\right)$, for $1\leq i \leq L$, and
	\begin{equation} \label{omp_hypothesis}
	\|\gama_L\|^s_{0,\infty} < \frac{1}{2}\left( 1+\frac{1}{\mu(\D^{(L)})} \right)-\frac{1}{\mu(\D^{(L)})}\cdot\frac{\epsilon_0}{|\gamma_{L}^{min}|},
	\end{equation}
	where $\gamma_{L}^{min}$ is the minimal entry in the support of $\gama_{L}$.
	Consider approximating the solution to the Pursuit step in Algorithm \ref{Alg:MulilayerPursuit} by running Orthogonal Matching Pursuit for $\|\gama_L\|_0$ iterations. Then
	\begin{enumerate}
		\item $Supp(\hat{\gama}_i) \subseteq Supp(\gama_i)$,
		\item $\|\hat{\gama}_i - \gama_i\|^2_2 \leq \frac{\mathcal{E}_0^2}{1-\mu(\D^{(L)})(\|\gama_L\|^s_{0,\infty}-1)} \left(\frac{3}{2}\right)^{L-i}$.
	\end{enumerate}
\end{customthm}

\begin{proof}
	Given that $\gama_L$ satisfies Equation \eqref{omp_hypothesis}, from \cite{WorkingLocallyThinkingGlobally} (Theorem 17) one obtains that 
	\begin{equation}
	\|\hat{\gama}_L - \gama_L\|^2_2 \leq \frac{\mathcal{E}_0^2}{1-\mu(\D^{(L)})(\|\gama_L\|^s_{0,\infty}-1)}.
	\end{equation}
	Moreover, if the OMP algorithm is run for $\|\gama_L\|_0$ iterations, then all the non-zero entries are recovered, i.e., $Supp(\hat{\gama}_L) = Supp({\gama}_L)$. Therefore, $\|\hat{\gama}_L - {\gama}_L\|^s_{0,\infty} \leq \| \gama_L\|^s_{0,\infty} = \lambda_L$.
	
	Now, let $\hat{\gama}_{L-1} = \D_{L}\hat{\gama}_L$. Regarding the support of $\hat{\gama}_{L-1}$, because $\gama_L$ satisfies the N.V.S. property, $\|\hat{\gama}_{L-1}\|_0 \leq \|\gama_{L-1}\|_0$. More so, all entries in $\hat{\gama}_{L-1}$ will correspond to non-zero entries in $\gama_{L-1}$. In other words,
	\begin{equation}
	Supp(\hat{\gama}_{L-1}) \subseteq Supp({\gama}_{L-1}).	
	\end{equation}
	
	Consider now the error at the $L-1$ layer, $\|\gama_{L-1}-\hat{\gama}_{L-1}\|_2^2$. Since $\| \gama_{L-1}-\hat{\gama}_{L-1} \|^s_{0,\infty} \leq \|\gama_{L-1}\|^s_{0,\infty}$, we can bound this error in terms of the Stripe RIP:
	\begin{equation}
	\|\gama_{L-1}-\hat{\gama}_{L-1}\|_2^2 = \|\D_L (\gama_{L}-\hat{\gama}_{L})\|_2^2 \leq (1+\delta_{\lambda_L})\|\gama_{L}-\hat{\gama}_{L}\|_2^2,
	\end{equation}
	We can further bound the SRIP constant as $\delta_k \leq (k-1)\mu(\D)$, from which one obtains 
	\begin{equation}
	\|\hat{\gama}_{L-1} - \gama_{L-1}\|^2_2 \leq \frac{\mathcal{E}_0^2 (1+(\|\gama_L\|^s_{0,\infty}-1)\mu(\D_L)) }{1-\mu(\D^{(L)})(\|\gama_L\|^s_{0,\infty}-1)}.
	\end{equation}
	From similar arguments, one obtains analogous claims for any $i^{th}$ layer; i.e., 
	\begin{multline}
	\|\hat{\gama}_{i} - \gama_{i}\|^2_2 \leq \frac{\mathcal{E}_0^2 }{1-\mu(\D^{(L)})(\|\gama_L\|^s_{0,\infty}-1)} \\ \prod_{j=i+1}^{L} (1+(\|\gama_j\|^s_{0,\infty}-1)\mu(\D_j)).
	\end{multline}
	
	This bound can be further relaxed for the sake of simplicity. Because $\|\gama_i\|^s_{0,\infty} < \frac{1}{2}\left(1+\frac{1}{\mu(\D_i)}\right)$, for $1\leq i \leq L$, then $(1+(\|\gama_L\|^s_{0,\infty}-1)\mu(\D_L)) < 3/2$, and so
	\begin{equation}
	\|\hat{\gama}_{i} - \gama_{i}\|^2_2 \leq \frac{\mathcal{E}_0^2 }{1-\mu(\D^{(L)})(\|\gama_L\|^s_{0,\infty}-1)} \left(\frac{3}{2}\right)^{L-i}.
	\end{equation}
\end{proof}

\subsection{Discussion on Theorem 6 and Theorem 7}
\label{app:DiscussionNVS}
In this section, we elaborate and comment further on the conditions impossed in the above theorems, regarding both the allowed sparsity and the N.V.S. property.
While the conditions of Theorems 6 and 7 might appear restrictive, the set of representations and dictionaries satisfying these conditions are not empty. An example of such constructions can be found in reference [17], where multi-layer overcomplete convolutional dictionaries are constructed by employing shifted versions of a discrete Meyer wavelet. This way, the resulting dictionaries have mutual coherence values in the order of $10^{-3}$ and $10^{-4}$, which provide ample room for sampling sparse representations satisfying the theorems assumptions. 

Regarding the NVS assumption, we stress that this is not as prohibitive as it might seem, and it is only needed because our theorems consider a deterministic worst-case scenario. Let us exemplify this better: consider a representation $\gama_2$ with 5 non-zero coefficients, and a dictionary $\D_2$ composed of atoms with 3 non-zeros each, uniformly distributed. If the entries in all non-zero coefficients are sampled from a normal distribution, the resulting inner representations $\gama_1 = \D_2\gama_2$ will have cardinalities in the range $[3, 15]$. If the mutual-coherence of $\D_1$ is such that the allowed maximal number of non-zeros per stripe (i.e., the $\ell_{0,\infty}$ norm) is, say, 7 (an assumption that is satisfied by the cases explained above), then this implies that the only signals that are allowed to exist are those \textbf{composed of atoms with some overlaps of their support}. The NVS assumption only implies that whenever these overlaps occur, they will not cancel each other. This, in fact, occurs with probability 1 if the non-zero coefficients are sampled from a Gaussian distribution.

We further depict this example in Figure \ref{app:Fig_NVS}. Note how the number of non-zeros in $\gamma_1$ is not allowed to be as large as possible (i.e., it is constrained to be below 7 by means of overlapping supports). The NVS property simply assumes that the coefficients multiplying $\d_2$ and $\d_3$ will not be such that the entry marked with red dotted line is zero. 

\begin{figure}[h]
	\centering
	\includegraphics[width = .45\textwidth]{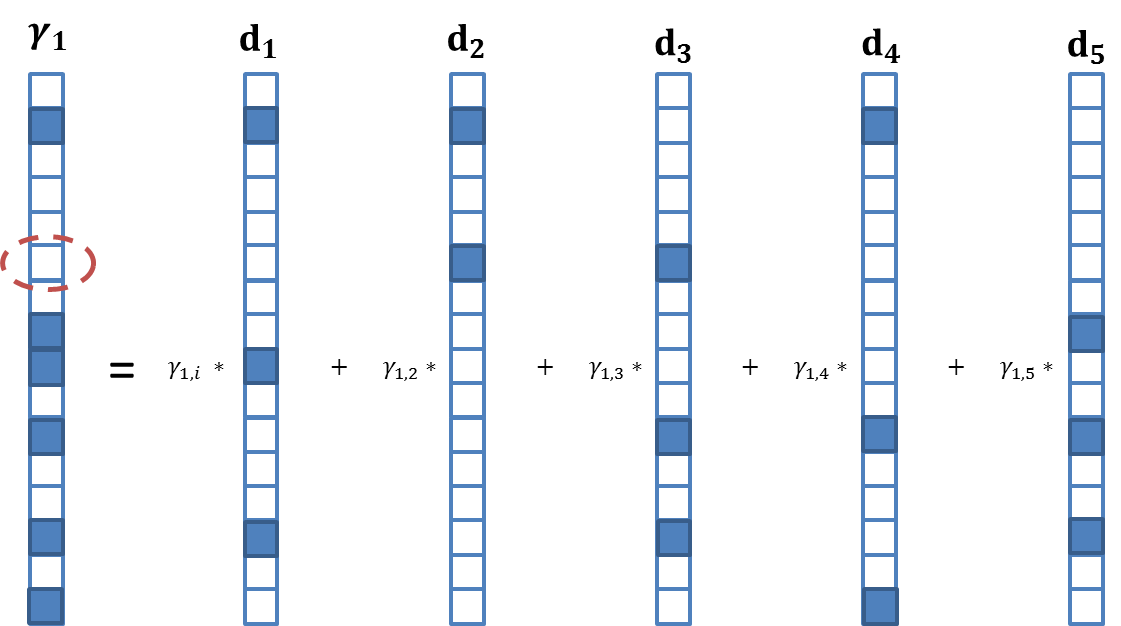}
	\caption{Illustration of the propagation of supports accross representations. See comments in the text.}
	\label{app:Fig_NVS}
\end{figure}

\subsection{Projecting General Signals}
\label{app:SketchProofProjection}
The method depicted in Algorithm \ref{Alg:ProjectionAlgorithm} can be shown to be a greedy approximation to an optimal algorithm, under certain assumptions, and we now provide a sketch of the proof of this claim. Consider the first iteration of the above method, where $k=1$. If OMP succeeds in providing the closest $\hat{\gama}_L$ subject to the respective constraint, i.e. providing the solution to 
\begin{equation}
\min_\gama \|\y - \D^{(L)} \gama\|^2_2 \text{ s.t. } \|\gama\|^s_{0,\infty} \leq 1,
\end{equation}
and if $\|\hat{\gama}_i\|^s_{0,\infty} \leq \lambda_i$ for every $i$, then this solution effectively provides the closest signal to $\y$ in the model defined by $\lamda = [\lambda_1, \dots, 1]$. If $\lambda_L = 1$, we are done.
Otherwise, if $\lambda_L > 1$, we might increase the number of non-zeros in $\hat{\gama}_L$, while decreasing the $\ell_2$ distance to $\y$. This is done by continuing to the next iteration: running again OMP with the constraint $\|\hat{\gama}_L\|^s_{0,\infty}\leq 2$, and obtaining the respective $\hat{\gama}_i$. 

At any $k^{th}$ iteration, due to the nature of the OMP algorithm, $Supp(\hat{\gama}^{k-1}_L) \subseteq Supp(\hat{\gama}^{k}_L)$. If all estimates $\hat{\gama}_i$ satisfy the N.V.S. property for the respective dictionaries $\D_i$, then the sparsity of each $\hat{\gama}_i$ is non-decreasing through the iterations, $\|\hat{\gama}^{k-1}_i\|^s_{0,\infty} \leq \|\hat{\gama}^{k}_i\|^s_{0,\infty}$. For this reason, if an estimate $\hat{\gama}^k_L$ is obtained such that any of the corresponding $\Loi$ constraints is violated, then necessarily one constraint will be violated at the next (or any future) iteration. Therefore, the algorithm outputs the signal corresponding to the iteration before one of the constraints was violated. A complete optimal (combinatorial) algorithm would need to retrace its steps and replace the last non-zero added to $\hat{\gama}^k_L$ by the second best option, and then evaluate if all constraints are met for this selection of the support. This process should be iterated, and Algorithm \ref{Alg:ProjectionAlgorithm} provides a greedy approximation to this idea.

\subsection{Sparse Dictionaries}
\label{app:SparseDictioanries}
\begin{customlemma}{3}{Dictionary Sparsity Condition} \\
	Consider the ML-CSC model $\M_\lamda$ described by the the dictionaries $\{\D_1\}_{i=1}^L$ and the layer-wise $\ell_{0,\infty}$-sparsity levels $\lambda_1,\lambda_2,\dots,\lambda_L$. Given $\gama_L : \|\gama_L\|^s_{0,\infty} \leq \lambda_L$ and constants $c_i = \Bigl\lceil \frac{2n_{i-1}-1}{n_i} \Bigr\rceil$, the signal $\x = \D^{(L)} \gama_L \in \M_\lamda$ if
	\begin{equation}
	\|\D_i\|_0 \leq \frac{\lambda_{i-1}}{\lambda_i c_i}, \quad \forall\ 1<i\leq L.
	\end{equation}
\end{customlemma}

\begin{proof}
	This lemma can be proven simply by considering that the patch-wise $\ell_{0,\infty}$ of the representation $\gama_{L-1}$ can be bounded by $\|\gama_{L-1}\|^p_{0,\infty} \leq \|\D_L\|_0 \|\gama_{L}\|^s_{0,\infty}$. Thus, if $\|\D_L\|_0 \leq \lambda_{L-1} / \lambda_{L}$ and $ \|\gama_{L}\|^s_{0,\infty} \leq \lambda_L$, then $\|\gama_{L-1}\|^p_{0,\infty} \leq \lambda_{L-1}$. Recalling the argument in \cite{Papyan2016convolutional} (Section 7.1), a stripe from the $i^{th}$ layer includes up to $c_i = \lceil (2n_{i-1}-1)/n_i \rceil$ patches. Therefore, $\|\gama_{L-1}\|^s_{0,\infty} \leq c_L \|\gama_{L-1}\|^p_{0,\infty}$, and so $\gama_{L-1}$ will satisfy its corresponding sparsity constraint if $\|\D_L\|_0 \leq \lambda_{L-1} / (c_L \lambda_{L} )$. Iterating this argument for the remaining layers proves the above lemma.
	
\end{proof}

\end{document}